\documentclass[journal]{IEEEtran}
\usepackage{times}
\usepackage{epsfig}
\usepackage{graphicx}
\usepackage{amsmath}
\usepackage{amssymb}
\usepackage{mathtools}
\usepackage{algorithm}
\usepackage{algorithmic}
\usepackage{graphics}
\usepackage{epstopdf}
\usepackage{subfigure}
\usepackage{multirow}
\usepackage{booktabs}
\usepackage{url}
\usepackage{amsthm}

\usepackage{makecell}
\usepackage{ragged2e}

\usepackage{cite}
\newtheorem{theorem}{Theorem}
\renewenvironment{proof}{{\noindent\bfseries Proof.}}{\qed}

\begin{document}
%
\title{Towards Compact ConvNets via Structure-Sparsity Regularized Filter Pruning}
%
%
%

\author{Shaohui Lin,~\IEEEmembership{Student Member,~IEEE,}
        Rongrong Ji$^{*}$,~\IEEEmembership{Senior Member,~IEEE,}
        Yuchao Li, \\
        Cheng Deng,~\IEEEmembership{Member,~IEEE,} 
        and Xuelong Li,~\IEEEmembership{Fellow,~IEEE}
\thanks{S. Lin and Y. Li are with the Fujian Key Laboratory of Sensing and Computing for Smart City, School of Information Science and Engineering, Xiamen University, 361005, China.}
\thanks{R. Ji (Corresponding author) is with the Fujian Key Laboratory of Sensing and Computing for Smart City, School of Information Science and Engineering, Xiamen University, 361005, China, and Peng Cheng Laboratory, Shenzhen, 518055, China (e-mail: rrji@xmu.edu.cn).}
\thanks{C. Deng is with the School of Electronic Engineering, Xidian University, Xi'an 710071, China.}
\thanks{X. Li is with the School of Computer Science and Center for OPTical IMagery Analysis and Learning (OPTIMAL), Northwestern Polytechnical University, Xi'an 710072, China.}}

%
%

\markboth{IEEE TRANSACTIONS ON NEURAL NETWORKS AND LEARNING SYSTEMS}%
{Shell \MakeLowercase{\textit{et al.}}: Bare Demo of IEEEtran.cls for IEEE Journals}

\maketitle

\begin{abstract}
The success of convolutional neural networks (CNNs) in computer vision applications has been accompanied by a significant increase of computation and memory costs, which prohibits their usage on resource-limited environments such as mobile or embedded devices. 
To this end, the research of CNN compression has recently become emerging. 
In this paper, we propose a novel filter pruning scheme, termed structured sparsity regularization (SSR), to simultaneously speedup the computation and reduce the memory overhead of CNNs, which can be well supported by various off-the-shelf deep learning libraries.
Concretely, the proposed scheme incorporates two different regularizers of structured sparsity into the original objective function of filter pruning, which fully coordinates the global outputs and local pruning operations to adaptively prune filters. 
We further propose an Alternative Updating with Lagrange Multipliers (AULM) scheme to efficiently solve its optimization. AULM follows the principle of ADMM and alternates between promoting the structured sparsity of CNNs and optimizing the recognition loss, which leads to a very efficient solver ($2.5\times$ to the most recent work that directly solves the group sparsity based regularization).
Moreover, by imposing the structured sparsity, the online inference is extremely memory-light, since the number of filters and the output feature maps are simultaneously reduced.
The proposed scheme has been deployed to a variety of state-of-the-art CNN structures including LeNet, AlexNet, VGGNet, ResNet and GoogLeNet over different datasets. 
Quantitative results demonstrate that the proposed scheme achieves superior performance over the state-of-the-art methods. We further demonstrate the proposed compression scheme for the task of transfer learning, including domain adaptation and object detection, which also show exciting performance gains over the state-of-the-art filter pruning methods. 
\end{abstract}


%
\IEEEpeerreviewmaketitle

\begin{figure*}[t]
\centering
  \subfigure[The compete process to prune network]{
    \includegraphics[scale = 0.45]{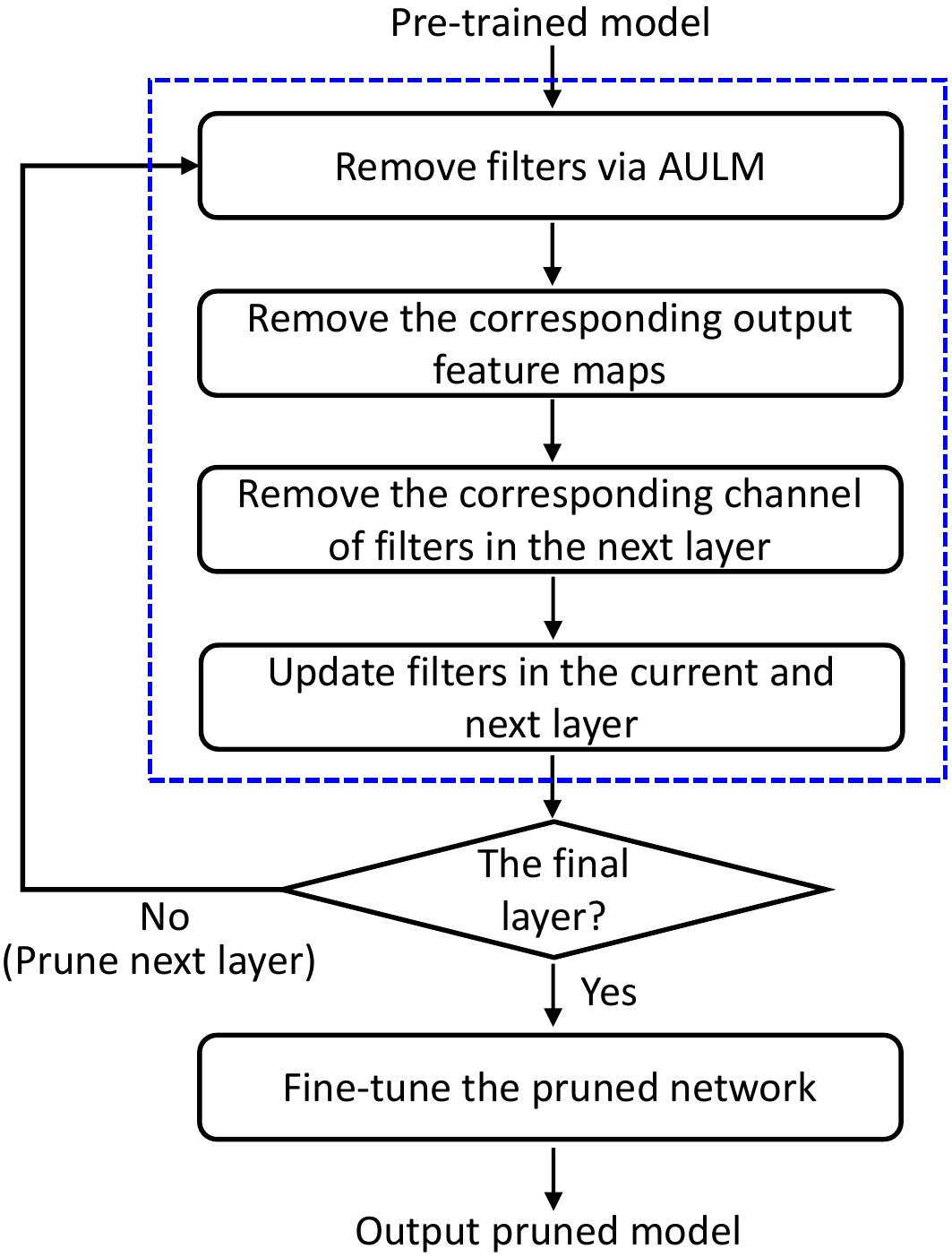}
    \label{fig:1_a}
    }
  \subfigure[The process to prune the single layer]{
    \includegraphics[scale = 0.55]{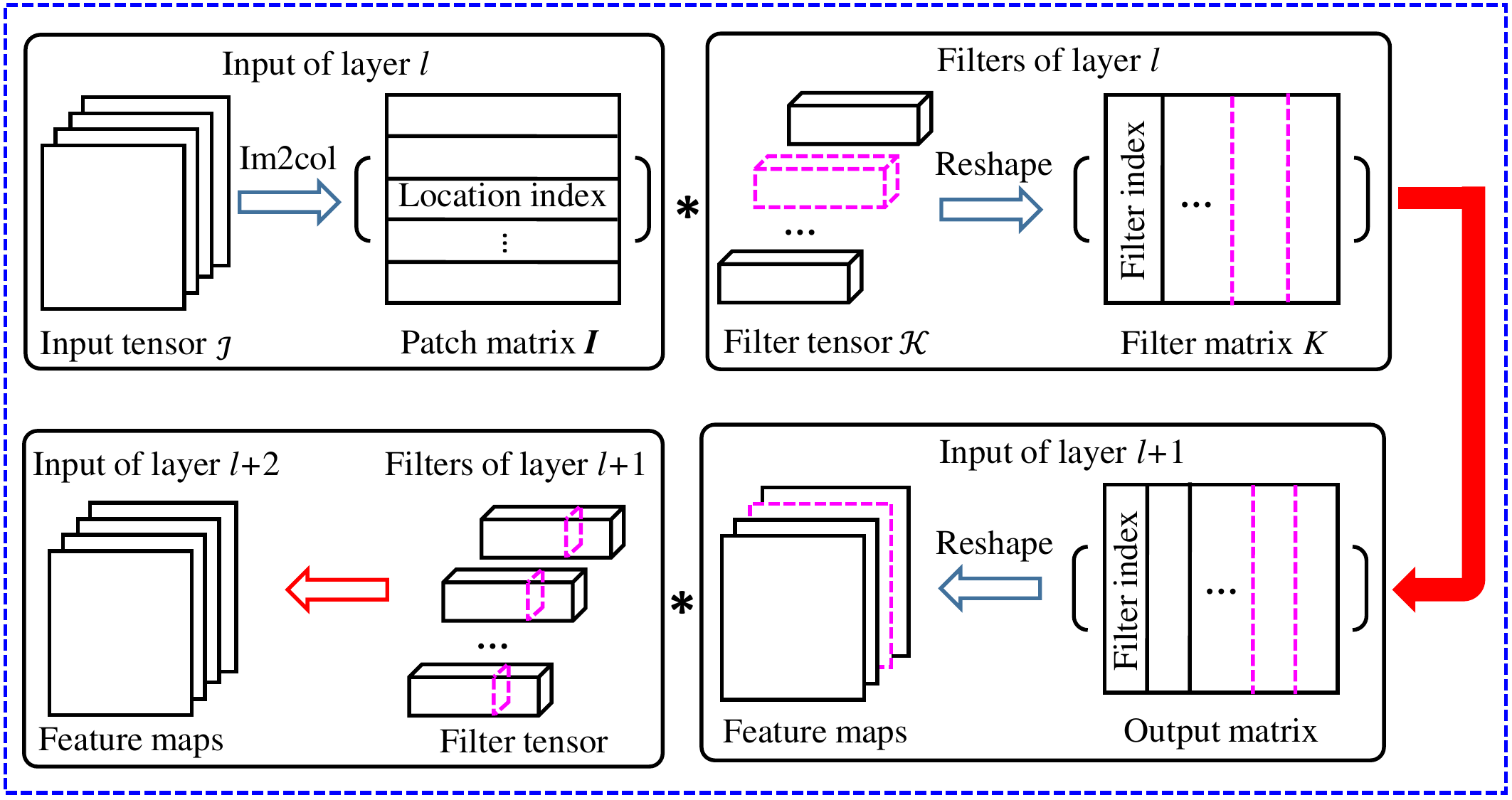}
    \label{fig:1_b}
  }
\vspace{-0.5em}
\caption{Illustration of SSR scheme. (a) The complete process includes adaptively pruning the unimportant filters via AULM solver, removing the corresponding output feature maps, removing the corresponding channels of filters in the next layer, as well as updating filters and fine-tuning the pruned network. (b) The process to select and prune unimportant filters, the corresponding output feature maps and channels in the next layer (highlighted in purple). In particular, the tensor-based convolutional operator can be replaced by a matrix-by-matrix multiplication using BLAS library to accelerate the computation of CNNs. (Best viewed in color.) }
\label{fig1}
\end{figure*}

\section{Introduction}
%
%
%
%
\IEEEPARstart{I}{N} recent years, convolutional neural networks (CNNs) have achieved great success on a variety of computer vision applications, ranging from action recognition \cite{chen2018deep}, object recognition \cite{simonyan2014very,he2016deep,krizhevsky2012imagenet,szegedy2015going,chollet2017xception,xie2017aggregated,huang2017densely},  to object detection \cite{girshick2015fast,girshick2014rich,ren2015faster,liu2016ssd}. 
One essential foundation of such success lies in the gigantic amount of model parameters to accompany with the large-scale training data. 
For instance, ResNet-152 \cite{he2016deep} has 57 million parameters, costs 230MB storage, and requires 11.3 billion FLOPs\footnote{FLOPs: The number of floating-point operations.} to classify one image with a resolution of $224 \times 224$.
Under such a circumstance, these models cannot be directly deployed to scenarios that require fast processing or compact storage, such as mobile systems or embedded devices. 

Substantial efforts have been devoted to the speedup and compression of CNNs for efficient online inference. 
Among the existing methods, network pruning has attracted increasing attention recently, due to its ability to reduce an overwhelming amount of parameters.
As one of the earliest works, LeCun \emph{et al}. \cite{lecun1989optimal} proposed an Optimal Brain Damage algorithm to prune the network by reducing the number of connections with a theoretically-justified saliency measurement. 
Later, Hassibi and Stork \cite{hassibi1993second} proposed an Optimal Brain Surgeon algorithm to remove unimportant parameters, which are determined by the second-order derivations of their weights. 
Recently, Han \emph{et al}. \cite{han2015deep,han2015learning} proposed to prune parameters with small magnitudes to reduce the network size. 
%
However, the above pruning schemes typically produce sparse CNNs with non-structured random connections, which will cause irregular memory access, \emph{i.e.}, a complex storage structure that adversely impacts the efficiency of accessing CNN models in memory. 
Moreover, such non-structured sparse CNNs cannot be supported by off-the-shelf libraries, which thus need specialized hardware \cite{han2016eie} or software \cite{liu2015sparse} designs to improve their efficiency in online inference.

To address the shortcoming of such non-structured connections, filter pruning is regarded as a promising solution, which has shown significant speedup in online inference as well as good independence to software/hardware platforms \cite{luo2017ThiNet,li2017pruning,anwar2015structured,lebedev2015fast,wen2016learning}. 
Differing from the previous works in parameter pruning, filter pruning can be further integrated by various CNN compression or acceleration methods, \emph{e.g.}, low-rank decomposition \cite{denil2013predicting,lebedev2014speeding,lin2017espace,lin2016towards,jaderberg2014speeding,tai2016convolutional,denton2014exploiting,kim2015compression}, DCT \cite{wang2016cnnpack} or FFT \cite{mathieu2013fast,vasilache2014fast} based frequency domain acceleration, and parameter quantization \cite{gong2014compressing,courbariaux2015binaryconnect,kim2015compression,rastegari2016xnor,cheng2018quantized}. 
Its another advantage lies in reducing the energy consumption \cite{yang2017designing}, which is not only influenced by the parameter amounts, but also influenced by the FLOPs and memory access of input/output feature maps. 
From this perspective, filter pruning can directly remove FLOPs and the intermediate activations of filters (\emph{e.g.}, output feature maps and their corresponding filter channels in the next layer), which substantially reduces the energy consumption. 

However, there are still open issues in the existing filter pruning schemes.
Concretely, how to adaptively and efficiently select important filters to reconstruct the \emph{global output}, \emph{i.e.}, the probabilistic ``softmax'' output after local filter pruning, is still an open problem, with only few works in the literature. 
For instance, the work in \cite{li2017pruning} proposed a magnitude based criterion to 
prune Convolutional filters with small $\ell_1$-norms. 
It has resulted in structured sparse patterns\footnote{Structured sparsity is to directly prune the filter/block (\emph{i.e.}, set the values of filter/block to zeros), while non-structured sparsity only justify whether each element in the filter is zero or not.} that can accelerate the online inference. 
However, such a magnitude-based measurement (\emph{e.g.}, $\ell_1$-norm) is too simple and inefficient to determine the importance of each filter, due to the existence of nonlinear activation functions (\emph{e.g.}, rectifier linear unit (ReLU) \cite{nair2010rectified}) and other complex operations (\emph{e.g.}, pooling and batch normalization \cite{ioffe2015batch}). 
To explain, filters with small $\ell_1$-norm values may have large responses in the output. For example, given two filter vectors $\mathbf{a}=(0.1,0.5,1)^\top$ and $\mathbf{b}=(0.2,0.1,0)^\top$ with an input $\mathbf{c}=(1,0,0)^\top$, we have $\|\mathbf{a}\|_1>\|\mathbf{b}\|_1$, while $\text{ReLU}(\mathbf{c}*\mathbf{a})<\text{ReLU}(\mathbf{c}*\mathbf{b})$ after convolutional operator and ReLU activation. 
Very recently, Luo \emph{et al.} \cite{luo2017ThiNet} implicitly associated the convolutional filter of each layer to its input channel of the next layer, upon which filter pruning is done by selecting input channels that have minimal local reconstruction error. 
The small reconstruction error, however, might be magnified and propagated in the deep networks, leading to large reconstruction error in the global outputs. 
A Taylor expansion based criterion \cite{molchanov2017pruning} was further proposed to iteratively prune one feature map and its associated filter. The pruned network is then fine-tuned to reduce the accuracy drop. However, such a scheme is unadaptive and costly in pruning the entire network. 
Group sparsity \cite{anwar2015structured,lebedev2015fast,wen2016learning,torfi2018attention,wang2018novel} was introduced to select unimportant filters by the Stochastic Gradient Descent (SGD). However, SGD is less efficient in convergence, which is also less efficient to generate the structured output of the pruned filters.  

 
In this paper, we propose an efficient filter pruning scheme, termed \emph{Structured Sparsity Regularization} (SSR), which can efficiently and adaptively prune a group of convolutional filters to minimize the classification error of the global output. 
Compared to the existing works of sequential filter pruning \cite{luo2017ThiNet,li2017pruning,wen2016learning,molchanov2017pruning,hu2016network}, we incorporate the structured sparsity constraint into the objective function of global output to model the correlation between the global output loss and the local filter removal, which produces a structured network with fast computing and light memory consumption\footnote{To make a fair comparison, we only evaluate the computational cost of model for online inference, without including the training}. In particular, we propose two different kinds of structured sparse regularizer for adaptive filter pruning, \emph{i.e.,} $\ell_{2,1}$-norm \cite{yuan2006model} and $\ell_{2,0}$-norm. 
The $\ell_{2,0}$-norm of a matrix $\mathbf{A}$ is defined as $\|\mathbf{A}\|_{2,0} = \sum_i\|\sqrt{\sum_j\mathbf{A}_{ij}^{2}}\|_0$, where for a scalar $a$, $\|a\|_0=0$ if $a=0$, and $\|a\|_0=1$ otherwise.
The $\ell_{2,1}$-norm is a convex norm to approximate the cardinality in filter selection, while the $\ell_{2,0}$-norm selects filter with the constraint of explicit cardinality, which is the most natural constraint. As for group sparsity with $\ell_{2,0}$-regularization, SGD in the previous works \cite{anwar2015structured,lebedev2015fast,wen2016learning,torfi2018attention} cannot solve the NP-hard problem for effective filter pruning.
To this end, we propose a novel Alternative Updating with Lagrange Multipliers (AULM), which handles the convergence difficulty with $\ell_{2,1}$-norm using SGD and the NP-hard problem with $\ell_{2,0}$-norm. AULM follows the principle of ADMM \cite{boyd2011distributed} by splitting the optimization problem into tractable sub-problems that can be solved efficiently. In addition, AULM has a faster convergence than ADMM by adding a Nesterov's optimal method \cite{nesterov1983method}, can effectively identify the importance of filters, and then updates parameters by alternating between promoting the structured sparsity and optimizing the recognition loss. 
In particular, the proposed solver circumvents the sparse constraint evaluations during the standard back-propagation step, which makes the implementation very practical. 
Moreover, compared to SSL \cite{wen2016learning}, the proposed solver is much faster to the existing solvers in offline pruning, \emph{i.e.}, almost $2.5\times$ than directly solving the regularizer with $\ell_{2,1}$-norm by using SGD. 
Fig.\,\ref{fig1} shows the proposed filter pruning framework.

Quantitatively, we demonstrate the advantage of the proposed SSR scheme using five widely-used models (\emph{i.e.}, LeNet, AlexNet, VGGNet-16, ResNet-50 and GoogLeNet) on two datasets (\emph{i.e.}, MNIST and ImageNet 2012). 
Compared to several state-of-the-art filter pruning methods \cite{luo2017ThiNet,li2017pruning,wen2016learning,molchanov2017pruning,torfi2018attention,hu2016network,yoon2017combined}, the proposed scheme performs drastically better, \emph{i.e.}, $5.5\times$ CPU speedup and $39.1\times$ compression with a negligible classification accuracy loss for LeNet on MNIST, $2.1\times$ CPU speedup with an increase of 1.28\% Top-5 classification error for AlexNet, $2.4\times$ GPU speedup with an decrease (even better) of 1.65\% Top-1 classification error for VGG-16, $1.9\times$ CPU speedup and $2.13\times$ compression with an increase of 3.65\% Top-1 classification error for ResNet-50, and $1.6\times$ CPU speedup and $1.7\times$ compression with an increase of 1.05\% Top-5 classification error for GoogLeNet in ImageNet 2012. 
Moreover, the pruned AlexNet and VGG-16 can be further compressed by replacing the original fully-connected layers with global average pooling \cite{lin2014network}, leading to $24.96\times$ and $15\times$ compression rates only with increases of 4.62\% and 0.27\% Top-5 classification error, respectively.

In addition, we also explore the generalization ability of SSR-based compressed model in more complex tasks, \emph{i.e.}, domain adaptation and object detection. Experimental results demonstrate that such a compressed model has achieved 3.5$\times$ FLOPs reduction and 15.4$\times$ compression only with an increase of 1.95\% Top-1 error on the task of domain adaptation, as well as 0.3\% mAP drops with a factor of 2.45$\times$ GPU speedup on the task of object detection. These two results are highly competitive comparing to the state-of-the-art filter pruning methods.

\section{Related Work}
Early works in network compression mainly focus on compressing the fully-connected layers \cite{lecun1989optimal,hassibi1993second,srinivas2015data,han2015deep,han2015learning}. For instance, LeCun \emph{et al.} \cite{lecun1989optimal} and Hassibi \emph{et al.} \cite{hassibi1993second} proposed a saliency measurement by computing the Hessian matrix of the loss function with respect to the parameters, based on which network parameters with low saliency values are pruned. 
Srinivas and Babu \cite{srinivas2015data} explored the redundancy among neurons to remove a subset of neurons without retraining. 
Han \emph{et al.} \cite{han2015deep,han2015learning} proposed a pruning scheme based on low-weight connections to reduce the total amount of parameters in CNNs. 
However, these methods only reduce the memory footprint and do not guarantee to reduce the computation time, since the time consumption is mostly dominated by the convolutional layers. 
Moreover, the above pruning schemes typically produce non-structured sparse CNNs that lack flexibility to be applied across different platforms or libraries. 
For example, the Compressed Sparse Column (CSC) based weight formation has to change the original format of weight storing in Caffe \cite{jia2014caffe} after pruning, which cannot be well supported across different platforms.

To reduce the computation cost of convolutional layers, a popular solution is to decompose convolutional filters into a sequence of tensors with fewer parameters \cite{denton2014exploiting,jaderberg2014speeding,kim2015compression,lebedev2014speeding,lin2017espace,huang2018ltnn}. 
The convolution can be also conducted in the frequency domain using DCT \cite{wang2016cnnpack} and FFT \cite{mathieu2013fast,vasilache2014fast}, or approximated by balanced decoupled spatial convolution \cite{xie2019balanced} to reduce the redundancy of spatial and channel information. Besides, binarization of weights \cite{courbariaux2015binaryconnect,courbariaux2016binarynet,rastegari2016xnor} and low-complexity weights \cite{cintra2018low} can be also employed in the convolutional layers to reduce the computation overheads with multiplication-free operations. 
Designing a compact filter is also able to accelerate the convolutional computation by replacing the over-parametric filters with a compact block, such as inception module in GoogLeNet \cite{szegedy2015going}, bottleneck module in ResNet \cite{he2016deep}, fire module in SqueezeNet \cite{SqueezeNet}, group convolution \cite{krizhevsky2012imagenet,zhang2018shufflenet,huang2018condensenet}, depth-wise separable convolution \cite{howard2017mobilenets,sandler2018mobilenetv2,chollet2017xception}. 
Without incurring additional overheads, our scheme can be integrated with the above schemes to further speedup the computation, since the above schemes are orthogonal to the core contribution of this paper.

In line with our work, some recent works have investigated structured pruning to remove redundant filters or feature maps, which can be categorized into either greedy based pruning \cite{li2017pruning,molchanov2017pruning,luo2017ThiNet,he2018soft,lin2018accelerating} or sparsity regularization based pruning \cite{anwar2015structured,lebedev2015fast,wen2016learning,torfi2018attention,yoon2017combined,liu2017learning}. 
For the former group, the work in \cite{li2017pruning} proposed a magnitude based pruning to prune filters with their corresponding feature maps by measuring the $\ell_1$-norm of filters, which is however inefficient in determining the importance of filters. He \emph{et al.} \cite{he2018soft} proposed an $\ell_2$-norm criterion to prune unsalient filters in a soft manner. 
Luo \emph{et al.} \cite{luo2017ThiNet} explored the importance of the input channel from the next convolutional layer, based on which conducted a local channel selection to prune unimportant input channels and the corresponding filters in the current layer. 
However, the small local reconstruction error might lead to large error in the global output by propagating through the deep network. 
A Taylor expansion based criterion was proposed in \cite{molchanov2017pruning} to iteratively prune one filter and then fine-tune the pruned network, which is however prohibitively costly for deep networks. 
Lin \emph{et al.} \cite{lin2018accelerating} proposed a global and dynamic pruning scheme to reduce redundant filters by greedy alternative updating.
Alternatively, group sparsity based regularization was proposed in \cite{anwar2015structured,lebedev2015fast,wen2016learning,torfi2018attention,wang2018novel} to penalize unimportant parameters and prune redundant filters directly by using SGD, which is also very slow in convergence for filter selection. 
To reduce redundancies in the model parameters, the combination of group and exclusive sparsity was proposed in \cite{yoon2017combined} to promote sharing and competition for features, respectively.
Different from these group sparsity based regularization, we investigate the structured sparsity of filters instead, including $\ell_{2,1}$-regularization and $\ell_{2,0}$-regularization. 
And the proposed AULM solver alternates the updating between promoting the structured sparsity and optimizing the recognition loss. By this way, AULM effectively overcomes the difficulty of convergence with $\ell_{2,1}$-regularization on SGD, and also can solve the NP-hard problem with $\ell_{2,0}$-regularization. 
Quantitatively, such an innovation has achieved much faster convergence and generated much more structured filters during training.
Recently, Liu \emph{et al.} \cite{liu2017learning} have proposed a network slimming scheme to associate a scaling factor in batch normalization with each filter channel, and imposed $\ell_1$-reguralization on these scaling factors to identify and prune unimportant channels. Different from network slimming, we directly focus on structured filter sparsity by $\ell_{2,1}$-regularization and $\ell_{2,0}$-regularization for pruning the complete filters.

\section{Structured Pruning via SSR}
In this section, we first describe the notations and preliminaries. Next, we present the general framework of structured filter pruning. Then, we present the proposed structured sparsity regularization scheme. Afterwards, AULM base solver is presented to perform the corresponding optimization. Finally, we discuss how to deploy our pruning strategy on the residual networks.

\subsection{Notations and Preliminaries}
Consider a CNN model consisting of $L$ layers in total (including convolutional and fully-connected layers), which are interlaced with rectifier linear units and pooling.
For the convolution operation, an input tensor $\mathcal{I}^l$ of size $H_l\times W_l\times C_l$ is transformed into an output tensor $\mathcal{O}^l$ of size $H_{l}^{'}\times W_{l}^{'}\times C_{l+1}$ by the following linear mapping at the $l$-th layer:
\begin{equation}
\label{eq1}
\mathcal{O}_{h',w',n}^{l} = \sum\limits_{i = 1}^{d_l}\sum\limits_{j=1}^{d_l}\sum\limits_{c=1}^{C_l}\mathcal{K}_{i,j,c,n}^{l} \mathcal{I}_{h_i,w_j,c}^{l},
\end{equation}
where the convolutional filter $\mathcal{K}^l$ at the $l$-th layer is a tensor of size $d_l\times d_l\times C_l\times C_{l+1}$. The spatial location of its output are denoted as $h'=h_i-i+1$ and $w'=w_j-j+1$, respectively. For simplicity, we assume a unit stride without zero-padding and skip biases. 

In practice, many deep learning frameworks (\emph{e.g.}, Caffe \cite{jia2014caffe} and Tensorflow \cite{abadi2016tensorflow}) compute tensor-based convolutional operator by a highly optimized matrix-by-matrix multiplication using linear algebra packages, such as Intel MKL and OpenBLAS. 
For example, an input tensor of size $H_l \times W_l \times C_l$ can be transformed into an input patch matrix $\mathbf{I}^l$ of size $(d_l\times d_l\times C_l)\times H_{l}^{'}W_{l}^{'}$ using \emph{im2col} operator. 
The columns of $\mathbf{I}^l$ are patch elements of the input tensor with size $d_l\times d_l\times C_l$. 
Correspondingly, convolutional filter is transformed into a filter matrix $\mathbf{K}^l$ of size $C_{l+1}\times (d_l\times d_l\times C_l)$ using \emph{reshaped} operator. 
Then, the output tensor can be obtained by reshaping the result matrix of size $C_{l+1}\times H_{l}^{'}W_{l}^{'}$, which is the result of multiplying filter matrix $\mathbf{K}^l$ with input patch matrix $\mathbf{I}^l$.
In this paper, we use Tensorflow to train and test our structured sparse CNNs. Therefore, we replace tensor-based filters $\mathcal{K}^l$ with matrix-based $\mathbf{K}^l$. 

In addition, we consider several norms of filter matrix $\mathbf{K}^l$, which are used in the regularization term. For example, the Frobenius norm of filter matrix $\mathbf{K}^l$ is defined as $\|\mathbf{K}^l\|_F\coloneqq\sqrt{\sum_{i,j}\mathbf{K}_{ij}^{l^2}}$. The sparsity inducing $\ell_1$-norm is defined as $\|\mathbf{K}^l\|_{1,1}\coloneqq\sum_{i=1}^{C_{l+1}}\|\mathbf{K}_i^l\|_1$. In this paper, we introduce two different structured sparsity norms to adaptively select unimportant filters to be pruned, \emph{i.e.}, $\ell_{2,1}$-norm and $\ell_{2,0}$-norm, which are denoted as $\|\mathbf{K}^l\|_{2,1}\coloneqq\sum_{i=1}^{C_{l+1}}\|\mathbf{K}_i^l\|_2$ and $\|\mathbf{K}^l\|_{2,0} = \sum_i^{C_{l+1}}\|\sqrt{\sum_j^{C_l}\mathbf{K}_{ij}^{l^2}}\|_0$, respectively. Note that $\ell_{2,0}$-norm is not a valid norm because it does not satisfy the positive scalarbility: $\|\alpha\mathbf{K}^l\|_{2,0} = |\alpha|\|\mathbf{K}^l\|_{2,0}$ for any scalar $\alpha$. The term ``norm'' here is for convenience.

\subsection{The framework of SSR}
SSR prunes the least important filters from a trained convolutional network to reduce the computation and memory costs. 
Its procedure consists of three basic operations, \emph{i.e.}, (1) evaluate the importance of each filter, (2) prune unimportant filters, and (3) fine-tune the whole network.
Differing from the previous filter pruning, we adaptively select unimportant filters to be pruned by using AULM. 
As shown in Fig.\,\ref{fig1}, We focus on the blue dotted boxes that perform AULM solver to adaptively prune unimportant filters, and then remove the corresponding output feature maps and the filter channels in the next layer. 
We present the principle steps of SSR as below:
\begin{itemize}
\item[1.] \emph{Automatic filter selection.} We design a novel objective function, which incorporates the structured sparsity constraint into the data error term, \emph{e.g.}, cross-entropy between the probability of data inference and the ground truth. The optimization problem can be solved by the proposed AULM solver. Thus, the unimportant filters are adaptively identified during the training.
\item[2.] \emph{Pruning.} We prune unimportant filters and their corresponding feature maps, together with the channels of filters in the next layer.
\item[3.] \emph{Updating.} We update the rest filters and feature maps in the current layer, as well as the channels of filters in the next layer.
\item[4.] \emph{Iteration of Step 1 to prune the next layer.}
\item[5.] \emph{Global fine-tuning}. We globally fine-tune the pruned networks, which recovers the discriminability and generalization ability for the pruned networks.
\end{itemize}

\subsection{The Objective Function of SSR}
Instead of directly pruning filters by calculating their corresponding magnitudes \cite{li2017pruning,hu2016network}, SSR utilizes the structured sparsity to seek a best trade-off between loss minimization and filter selection. We consider the following objective function:
\begin{equation}
\label{eq2}
\min_{\mathbf{K}} \mathcal{L}\big(\mathbf{Z}, f(\mathbf{X};\mathbf{K})\big)+\lambda  g(\mathbf{K}).
\end{equation}
Here $\mathbf{K}$ represents the collection of all weights in CNNs. 
$\mathcal{L}\big(\mathbf{Z}, f(\mathbf{X};\mathbf{K})\big)$ is the cross-entropy loss for classification or mean-squared error for regression between the labels of ground truth $\mathbf{Z}$ and the output of the last layer in CNNs $f(\mathbf{X};\mathbf{K})$\footnote{For simplicity, the term of weight decay (\emph{i.e.}, non-structured regularization applying on every weight, \emph{e.g.}, $\ell_2$-norm) is omitted, since it can be directly incorporated into the loss function and does not affect the result of structured sparsity regularization.}, where $\mathbf{D} = \big\{\mathbf{X}, \mathbf{Z}\big\} = \big\{\mathbf{X}_i, \mathbf{Z}_i\big\}_{i=1}^{N}$ is a training dataset with $N$ instances. We denote the first loss term in Eq.\,(\ref{eq2}) as $\mathcal{L}_D(\mathbf{K})$ for simplicity. 
The term $g(\mathbf{K})$ is certain structured sparsity regularization on the total size of remaining filters in each iteration. In this paper, we consider two different kinds of structured sparsity regularizer, \emph{i.e.,} $\ell_{2,1}$-norm and $\ell_{2,0}$-norm.

Parameter $\lambda$ is the penalty term of structured sparsity. As $\lambda$ varies, the solution of Eq.\,(\ref{eq2}) traces a trade-off path between the performance and structured sparsity. 
Note that the $\ell_{2,0}$ or $\ell_{2,1}$-norm based structured sparse regularizer $g(\mathbf{K})$ and its efficient solver in Eq.\,(\ref{eq2}) is non-trivial.

\subsection{The AULM solver}
The proposed AULM solver aims at solving the problem of structure-sparsity regularization, which is inspired by ADMM in the field of distributed optimization. Different from ADMM, we focus on selecting and pruning unimportant filters by solving a non-convex optimization problem with the $\ell_{2,1}$-regularization and an NP-hard problem with the $\ell_{2,0}$-regularization.
In particular, to handle the non-trival regularizer in Eq.\,(\ref{eq2}), we introduce a slack variable and an equality constraint as follows:
\begin{equation}
\label{eq3}
\begin{split}
& \min_{\mathbf{K},\mathbf{F}} \mathcal{L}_D(\mathbf{K})+\lambda g(\mathbf{F}) \\
& \emph{s.t.}\quad \mathbf{K} = \mathbf{F}.
\end{split}
\end{equation}
 
AULM is an iterative method that augments the Lagrangian function with quadratic penalty terms.
The augmented Lagrangian associated with the constrained problem of Eq.\,(\ref{eq3}) is given by:
\begin{small}
\begin{equation}
\label{eq4}
\begin{split}
&\mathcal{L}_{(\mathbf{K},\mathbf{F},\mathbf{Y})}=\mathcal{L}_D(\mathbf{K})+\sum\limits_{l=1}^{L}\lambda g(\mathbf{F}^l) \\
& \quad +\sum\limits_{l=1}^{L}\text{trace}\big(\mathbf{Y}^{l^\top}(\mathbf{K}^l-\mathbf{F}^l)\big) + \sum\limits_{l=1}^{L}\frac{\rho}{2}\|\mathbf{K}^l-\mathbf{F}^l\|_{F}^{2},
\end{split}
\end{equation}
\end{small}

\noindent where $\mathbf{K}^l, \mathbf{F}^l, \mathbf{Y}^l$ are the filter kernel, the intermediate filter with structured sparsity, and the dual variables (\emph{i.e.} the lagrange multipliers) at the $l$-th layer, respectively. $\rho > 0$ is a penalty parameter. To minimize Eq.\,(\ref{eq4}), AULM solves for each variable via a sequence of iterative computations:
\begin{itemize}
\item[1.] Employ gradient descent to minimize the loss of
\begin{equation}
\label{eq5}
\mathbf{K}^{\{k+1\}} = \mathop{\min}_{\mathbf{K}}\mathcal{L}\Big(\mathbf{K}, \hat{\mathbf{F}}^{\{k\}}, \hat{\mathbf{Y}}^{\{k\}}\Big).
\end{equation} 
\item[2.] Find the closed-form solution of the structured sparsity:
\begin{equation}
\label{eq6}
\mathbf{F}^{\{k+1\}} = \mathop{\min}_{\mathbf{F}}\mathcal{L}\Big(\mathbf{K}^{\{k+1\}}, \mathbf{F}, \hat{\mathbf{Y}}^{\{k\}}\Big).
\end{equation}
\item[3.] Update the dual variables $\mathbf{Y}^{l}$ using gradient ascent with a step-size equal to $\rho$, \emph{i.e.},
\begin{equation}
\label{eq7}
\mathbf{Y}^{\{k+1\}}=\hat{\mathbf{Y}}^{\{k\}}+\rho\Big(\mathbf{K}^{\{k+1\}}-\mathbf{F}^{\{k+1\}}\Big).
\end{equation}
\item[4.] Conduct an overrelaxation step for accelerated variables $\hat{\mathbf{F}}$ and $\hat{\mathbf{Y}}$ with a step-size equal to $\gamma$, \emph{i.e.},
\begin{equation}
\label{eq_tmp1}
\hat{\mathbf{Y}}^{\{k+1\}} = \mathbf{Y}^{\{k+1\}} + \gamma^{\{k+1\}}\Big(\mathbf{Y}^{\{k+1\}} - \mathbf{Y}^{\{k\}}\Big),
\end{equation}
\begin{equation}
\label{eq_tmp2}
\hat{\mathbf{F}}^{\{k+1\}} = \mathbf{F}^{\{k+1\}} + \gamma^{\{k+1\}}\Big(\mathbf{F}^{\{k+1\}} - \mathbf{F}^{\{k\}}\Big),
\end{equation}
\noindent where $\gamma^{\{k+1\}} = k/(k+r)$, with $r\geq 3$ ($r = 3$ is the standard choice). 
\end{itemize} 
\noindent The above four steps are applied in an alternating manner. Below we describe the details of step 1 and step 2 to obtain $\mathbf{K}$ and $\mathbf{F}$. The proposed alternative optimization is summarized in Alg.\,\ref{alg1}. By overrelaxing the Lagrange multiplier variables after each iteration, AULM not only has a faster convergence, but also obtains a more effective solution compared to ADMM.

\begin{algorithm}[t]
\small
\caption{AULM for structured pruning CNN}
\renewcommand{\algorithmicrequire}{\textbf{Input:}} 
\renewcommand{\algorithmicensure}{\textbf{Output:}}
\begin{algorithmic}[1]
\REQUIRE 
Training data points $\mathbf{D}$,
pre-trained CNN weights $\mathbf{K}$,
a set of regularization factors $\mathcal{S}$.
\ENSURE 
The structured pruning filters $\mathbf{K}$.\\
\STATE
Initialize: dual variables $\hat{\mathbf{Y}}=\mathbf{Y}=\textbf{0}$, $\hat{\mathbf{F}}=\mathbf{F}=\mathbf{K}$, and $\rho=1$.
\FOR{ each $\lambda$ in $\mathcal{S}$} 
\FOR{ each $l$ in $[1,L]$}
\REPEAT 
\STATE 
 \textbf{Step 1:} Find the estimation of $\mathbf{K}^{l^{\{k+1\}}}$ by solving the problem in Eq.\,(\ref{eq8_tmp}) using SGD;\\ 
\STATE
 \textbf{Step 2:} Find the structured sparsity estimation of $\mathbf{F}^{l^{\{k+1\}}}$ with $\ell_{2,1}$-norm or $\ell_{2,0}$-norm from Eq.\,(\ref{eq11}) or Eq.\,(\ref{eq12}), respectively;\\
\STATE 
 \textbf{Step 3:} Update dual variables $\mathbf{Y}^{l^{\{k+1\}}}$  by Eq.\,(\ref{eq7}).
\STATE 
 \textbf{Step 4:} Update accelerated variables $\hat{\mathbf{Y}}^{l^{\{k+1\}}}$ and $\hat{\mathbf{F}}^{l^{\{k+1\}}}$  by Eq.\,(\ref{eq_tmp1}) and Eq.\,(\ref{eq_tmp2}), respectively.
\UNTIL{$\|\mathbf{K}^{l^{\{k+1\}}}-\mathbf{F}^{l^{\{k+1\}}}\|_F\leq\epsilon$, or $\|\mathbf{F}^{l^{\{k+1\}}}-\mathbf{F}^{l^{\{k\}}}\|_F\leq\epsilon$.} 
\STATE
Prune filters $\mathbf{K}^l$ corresponding to the row-index of $\mathbf{F}^l$ with zeros and their corresponding feature maps.
\ENDFOR
\STATE
Fine-tune the pruned network.
\ENDFOR 
\end{algorithmic}
\label{alg1}
\end{algorithm}

\subsubsection{The Updating of AULM}
\label{ssub_3_4_1}
\textbf{Step 1. } By removing the penalty terms of $g(\mathbf{F}^l)$ and completing the squares with respect to $\mathbf{K}$ in Eq.\,(\ref{eq4}), we obtain the following equivalent problem to Eq.\,(\ref{eq5}):
\begin{small}
\begin{equation}
\label{eq8}
\min_{\mathbf{K}}\mathcal{L}_D(\mathbf{K})+\sum\limits_{l=1}^{L}\frac{\rho}{2}\|\mathbf{K}^l-\mathbf{T}_1^l\|_{F}^{2},
\end{equation}
\end{small}

\noindent where $\mathbf{T}_1^l=\mathbf{F}^l-\frac{1}{\rho}\mathbf{Y}^l$. To obtain the sub-optimal filters $\mathbf{K}$ in the layer-wise pruning framework, we separately update $\mathbf{K}^l$ in the $l$-th layer with the following optimization problem:
\begin{small}
\begin{equation}
\label{eq8_tmp}
\min_{\mathbf{K}^l}\mathcal{L}_D(\mathbf{K}^l)+\frac{\rho}{2}\|\mathbf{K}^l-\mathbf{T}_1^l\|_{F}^{2}.
\end{equation}
\end{small}

\noindent We use Stochastic Gradient Descent (SGD) to optimize the filters $\mathbf{K}^l$, which is a reasonable choice to handle such a high-dimensional optimization. The entire procedure relies mainly on the standard forward-backward pass. 

\textbf{Step 2. } By removing the first term $\mathcal{L}_D(\mathbf{K})$ and completing the squares with respect to $\mathbf{F}$ in Eq.\,(\ref{eq4}), we obtain the following equivalent problem to Eq.\,(\ref{eq6}):
\begin{equation}
\label{eq9}
\min_{\mathbf{F}}\sum\limits_{l=1}^{L}\lambda g(\mathbf{F}^l)+\sum\limits_{l=1}^{L}\frac{\rho}{2}\|\mathbf{F}^l-\mathbf{T}_2^l\|_{F}^{2},
\end{equation}

\noindent where $\mathbf{T}_2^l=\mathbf{K}^l+\frac{1}{\rho}\mathbf{Y}^l$. We update $\mathbf{F}^l$ layer-by-layer instead of directly updating the whole layers. Hence, we get the following optimization problem at the $l$-th layer:
\begin{equation}
\label{eq10}
\min_{\mathbf{F}^l}\lambda g(\mathbf{F}^l)+\frac{\rho}{2}\|\mathbf{F}^l-\mathbf{T}_2^l\|_{F}^{2}.
\end{equation}
Based on Eq.\,(\ref{eq10}), we can obtain a closed-form solution by considering the following regularizers $g(\mathbf{F}^l)$:
\begin{itemize}
\item\textbf{$\ell_{2,1}$-norm.} A closed-form solution of Eq.\,(\ref{eq10}) can be derived, which is evaluated row-by-row on $\mathbf{T}_2^l$ \cite{cotter2005sparse,goldstein2014field}. The $i$-th row is calculated via:
\begin{equation}
\label{eq11}
{\mathbf{F}_i^l} = {\mathbf{T}_2^l}_i\frac{\max\{\|{\mathbf{T}_2^l}_i\|_2-\frac{\lambda}{\rho}, 0\}}{\|{\mathbf{T}_2^l}_i\|_2}.
\end{equation}
\item\textbf{$\ell_{2,0}$-norm.} A closed-form solution of Eq.\,(\ref{eq10}) with this regularizer can be also derived, which is evaluated row-by-row on $\mathbf{T}_2^l$ in \textbf{Theorem \ref{the1}}. The $i$-th row is calculated via:
\begin{equation}
\label{eq12}
{\mathbf{F}_i^l}=\left \{ \begin{matrix}
\mathbf{0}, & \lambda\geq\frac{\rho}{2}\|{\mathbf{T}_2^l}_i\|_2^2\\
{\mathbf{T}_2^l}_i, & \lambda<\frac{\rho}{2}\|{\mathbf{T}_2^l}_i\|_2^2.\\
\end{matrix}
\right.
\end{equation}
\item\textbf{$\ell_{1}$-norm.} The norm is not a structural constraint, which leads to unstructured sparsity by solving the problem in Eq.\,(\ref{eq10})\footnote{Here, we consider the $\ell_1$-norm to better validate the effectiveness of simultaneously accelerating the computation and compressing the memory overhead of CNNs by the aforementioned structured sparsity.}. Specifically, we obtain a closed-form solution of Eq.\,(\ref{eq10}) with the constraint by evaluating on each entry of $\mathbf{T}_2^l$. The optimal solution ${F_{ij}^l}^{*}$ is obtained via: 
\begin{equation}
\label{eq13}
{F_{ij}^l} = sign({T_2^l}_{ij})\max\{{|T_2^l}_{ij}|-\lambda/\rho, 0\}.
\end{equation}
\noindent where $sign(\cdot)$ is an indicator function, \emph{i.e.}, 
\begin{displaymath}
sign({T_2^l}_{ij})=\left \{ \begin{matrix}
1, & {T_2^l}_{ij}\geq 0,\\
-1, & otherwise.\\
\end{matrix}
\right .
\end{displaymath}
\end{itemize}

\begin{theorem} 
\label{the1}
Let $g(\mathbf{F}^l)$ be a regularizer by $\|\mathbf{F}^l\|_{2,0}$, then the optimal solution of Eq.\,(\ref{eq10}) is given by Eq.\,(\ref{eq12}), where $\mathbf{F}^l = ({\mathbf{F}_1^{l}}, {\mathbf{F}_2^{l}}, \cdots, {\mathbf{F}_{C_{l+1}}^{l}})^{\top}$ and $\mathbf{T}_2^l = ({\mathbf{T}_2^l}_1, {\mathbf{T}_2^l}_2, \cdots, {\mathbf{T}_2^l}_{C_{l+1}})^{\top}$.
\end{theorem}

\begin{proof}
\label{pro1}
Since $g(\mathbf{F}^l) = \|\mathbf{F}^l\|_{2,0} = \sum_i\|\sqrt{\sum_j\mathbf{F}_{ij}^{l^2}}\|_0$, we are interested in solving the following problem:
\begin{equation}
\label{eq14}
\min_{\mathbf{F}^l}\lambda\sum_i\|\sqrt{\sum_j\mathbf{F}_{ij}^{l^2}}\|_0+\frac{\rho}{2}\|\mathbf{F}^l-\mathbf{T}_2^l\|_{F}^{2}.
\end{equation}
We rewrite $\mathbf{F}^l$ and $\mathbf{T}_2^l$ as 
$\mathbf{F}^l = ({\mathbf{F}_1^{l}}, {\mathbf{F}_2^{l}}, \cdots, {\mathbf{F}_{C_{l+1}}^{l}})^{\top}$ and $\mathbf{T}_2^l = ({\mathbf{T}_2^l}_1, {\mathbf{T}_2^l}_2, \cdots, {\mathbf{T}_2^l}_{C_{l+1}})^{\top}$,
where ${\mathbf{F}_i^{l}}$ and ${\mathbf{T}_2^l}_i$ are the $i$-th row of $\mathbf{F}^l$ and $\mathbf{T}_2^l$, respectively. Then we solve the following equivalent problem for each row independently to Eq.\,(\ref{eq10}):
\begin{equation}
\label{eq15}
\min_{\mathbf{F}_i^l}\mathcal{L}(\mathbf{F}_i^l)=\lambda\|\mathbf{F}_{i}^{l}\|_0+\frac{\rho}{2}\|{\mathbf{F}_i^{l}}-{\mathbf{T}_2^l}_i\|_{2}^{2}.
\end{equation}
On one hand, $\forall \mathbf{F}_i^l\neq\mathbf{0}, \lambda\|\mathbf{F}_{i}^{l}\|_0=\lambda$. We obtain the optimal value $\mathcal{L}(\mathbf{F}_i^l)=\lambda$, when $\|{\mathbf{F}_i^{l}}-{\mathbf{T}_2^l}_i\|_{2}^{2} = 0$, \emph{i.e.,} $\mathbf{F}_i^{l}={\mathbf{T}_2^l}_i$. 
On the other hand, if $\mathbf{F}_i^l=\mathbf{0}, \lambda\|\mathbf{F}_{i}^{l}\|_0 = 0$, and $\mathcal{L}(\mathbf{F}_i^l)=\frac{\rho}{2}\|{\mathbf{T}_2^l}_i\|_2^2$. Therefore, when $\lambda\geq\frac{\rho}{2}\|{\mathbf{T}_2^l}_i\|_2^2$, the optimal solution is $\mathbf{F}_i^l=\mathbf{0}$, while $\lambda<\frac{\rho}{2}\|{\mathbf{T}_2^l}_i\|_2^2$, the optimal solution is $\mathbf{F}_i^l={\mathbf{T}_2^l}_i$. Thus, we can obtain the optimal solution of Eq.\,(\ref{eq10}), which is given by Eq.\,(\ref{eq12}) by considering the row-wise decoupling property of Eq.\,(\ref{eq10}).

\end{proof}

\begin{figure}[!t]
\begin{center}
\includegraphics[scale=0.47]{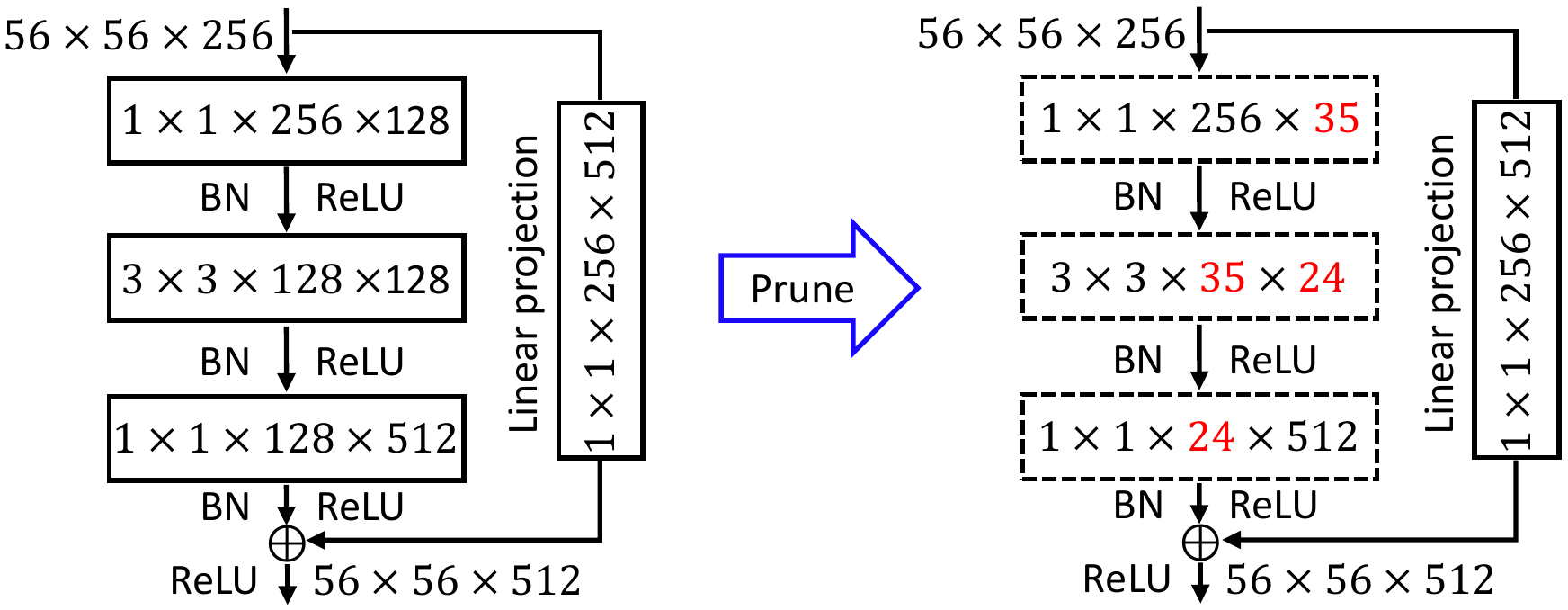}
\end{center}
\vspace{-1em}
\caption{Illustration of pruning ResNet. The red value is the number of remaining filters/channels.}
\label{fig2}
\end{figure}

\subsubsection{The Convergence of AULM}
Since non-linear transition, normalization and pooling commonly occur in CNNs, the objective function of Eq.\,(\ref{eq4}) is highly non-convex, which lacks a theoretical proof to guarantee its convergence to the global optimal. However, it is empirically shown that AULM works well when the penalty parameter $\rho$ is sufficiently large. This is related to the quadratic term that tends to be locally convex by giving a sufficiently large $\rho$. However, if $\rho$ is too large, it is difficult for the iterative solver to take effects. As a trade-off, we set $\rho=1$ in our implementation.

Since the objective function is highly non-convex, there is a risk to be trapped into a local optimum. In our implementation, we circumvent this difficulty by using the pre-trained weights as the initialized weights, which performs quite well in practice. 

\subsection{Pruning on ResNet}
\label{ssec_3_5}
Unlike VGG-16 and AlexNet, there are some restrictions to prune ResNet due to the special residual blocks. 
In general, each residual block with bottleneck structure contains three convolutional layers (followed by both batch normalization and ReLU) and shortcut connections. 
In order to perform the sum operator, the number of output feature maps in the last convolutional layer needs to be consistent with that of the projection shortcut layer. 
In particular, when the dimensions of input/output channels are mismatched in a residual block, a linear projection is performed by the shortcut connections (see \cite{he2016deep} for more details). 

In this paper, we focus on pruning the first two layers in each residual block, as shown in Fig.\,\ref{fig2}. And we do not prune the last convolutional layer of each residual block directly. 
In fact, the parameters (\emph{e.g.}, filter channel) in the last convolutional layer are much fewer, as a large proportion of filters in the second layer has been pruned.

\section{Experiments}
\subsection{Experimental Setups}
\label{ssec_4_1}
\textbf{Models and datasets. }
We conduct comprehensive experiments using five convolutional networks on two datasets, \emph{i.e.}, LeNet on MNIST \cite{lecun1998gradient}, AlexNet \cite{krizhevsky2012imagenet}, VGG-16 \cite{simonyan2014very}, ResNet-50 \cite{he2016deep}  and GoogLeNet \cite{szegedy2015going} on ImageNet \cite{ILSVRC15}. 
We implement the proposed SSR scheme with Tensorflow\footnote{Our source code is available at https://github.com/ShaohuiLin/SSR.}\cite{abadi2016tensorflow}. 
All pre-trained CNNs except LeNet are taken from the Caffe model zoo\footnote{https://github.com/BVLC/caffe/wiki/Model-Zoo}. We make use of an open source tool\footnote{https://github.com/ethereon/caffe-tensorflow} to convert the pre-trained models to Tensorflow format\footnote{The accuracies of models may be slight different from that reported by other works, due to a different learning framework.} and then fine-tune them to restore the accuracy. We train LeNet from scratch and report the results in Table\,\ref{tab1}.

\textbf{Implementations. }
To train the proposed SSR scheme, we use a learning rate of 0.001 with a constant dropping factor of 10 throughout 10 epochs. The weight decay is set to be 0.0005 and the momentum is set to be 0.9. To train SSR on both LeNet and AlexNet, the mini-batch size is set to be 256. To train VGG-16, ResNet-50 and GoogLeNet, the mini-batch size is all set to be 32. 
After pruning, the pruned network is fine-tuned for 30 epochs, in which a learning rate starts at $10^{-4}$ and is scaled by 0.1 throughout 10 epochs. All experiments are run on NVIDIA GTX 1080Ti graphics card with 11GB and 128G RAM.
The number of pruned filters is directly controlled by the hyper-parameters $\lambda$, \emph{i.e.}, the regularization factor of the structured sparsity. In our experiments, we vary $\lambda$ in the set of $\{0.1, 0.2\cdots, 0.8\}$ with 8 values to select the best trade-off between the compression/speedup rate and the accuracy. 
For $r$ in the overrelaxation step, we set it to be 3.

\begin{table}[t]
\scriptsize
\begin{center}
\caption{Pruning results of LeNet on MNIST. ``Num-Num-Num'' is the number of remaining filters in each layer. K/M/B means thousand/million/billion in this paper, respectively. The testing mini-batch size is set to be 100.}
\label{tab1}
\vspace{-1.5em}
\begin{tabular}{c|ccc|ccc}
\Xhline{0.1em}
\multirow{2}*{Method} & \multirow{2}*{\#Filter/Node} & \multirow{2}*{FLOPs} & \multirow{2}*{\#Param.} & \multicolumn{1}{c}{CPU} & \multirow{2}*{Speedup} & \multicolumn{1}{c}{Top-1}\\
& & & & (ms) & & Err. $\uparrow$ \\
\Xhline{0.1em}
LeNet & 20-50-500 & 2.3M & 0.43M & 26.4 & $1\times$ & 0\% \\ \hline
\multirow{2}{*}{SSL \cite{wen2016learning}}  & 3-15-175 & 162K & 45K & 7.3 & $3.62\times$ & 0.05\% \\
& 2-11-134 & 91K & 26K & 6.0 & $4.40\times$ & 0.20\% \\  \hline
\multirow{2}{*}{TE \cite{molchanov2017pruning}} & 2-12-127 & 95K & 27K & 5.7 & $4.63\times$ & 0.02\% \\
 & 2-7-99 & 65K & 13K & 5.5 & $4.80\times$ & 0.20\%\\ \hline
CGES \cite{yoon2017combined} & - & 332K & 156K & - & - & 0.01\% \\
CGES+ \cite{yoon2017combined} & - & - & 43K & - & - & 0.04\% \\ \hline
\multirow{2}{*}{GSS \cite{torfi2018attention}} & 3-11-109 & 119K & 21K & 6.7 & $3.94\times$ & 0.08\% \\
& 3-8-82 & 95K & 12K & 5.6 & $4.71\times$ & 0.20\% \\ \hline
\multirow{2}{*}{SSR-L1} & - & - & 16K & 23.6 & $1.12\times$ & 0\% \\
& - & - & 9K & 20.2 & $1.31\times$ & 0.12\% \\ \hline
\multirow{2}{*}{SSR-L2,1} & 3-11-108 & 118K & 21K & 6.6 & $4.00\times$ & 0.05\% \\
& 2-8-77 & 67K & 11K & 4.8 & $5.50\times$ & 0.18\% \\ \hline
\multirow{2}{*}{SSR-L2,0} & 2-11-146 & 93K & 28K & 5.6 & $4.71\times$ & 0.05\% \\
& 2-8-79 & 67K & 11K & 4.9 & $5.39\times$ & 0.20\% \\
\Xhline{0.1em}
\end{tabular}
\end{center}
\end{table}

\textbf{Evaluation Protocols. }
For evaluation protocols, we quantize the performance by using FLOPs, the number of parameters and the Top-1/5 classification error.
To make a fair comparison, the speedup rate is measured in a single-thread Intel Xeon E5-2620 CPU and NVIDIA GTX TITAN X GPU.

\textbf{Alternative and State-of-the-art Approaches. }We first compare our filter selection criterion with three alternative criteria, which are briefly summarized as follows:
\begin{itemize}
\item[1.] \textbf{Random.} Randomly prune filters of each layer.
\item[2.] \textbf{L1-norm (Filter norm)} \cite{li2017pruning}. Filters with smaller magnitude tend to be unimportant filters. Therefore, the $\ell_1$-norm of each filter $s_i = \|\mathbf{K}(i,:)\|_1$ is chosen as its importance score. We then sort these scores to prune filters correspondingly.
\item[3.] \textbf{APoZ (Average Percentage of Zeros)} \cite{hu2016network}. The sparsity of each channel in output after ReLU activation can be chosen as the importance score of the corresponding filter. Then the sparsity is calculated as $s_i =1-\frac{1}{|\mathcal{I}(:,:,i)|}\sum\sum\mathbb{I}\big(\mathcal{I}(:,:,i) == 0\big)$, where $|\mathcal{I}(:,:,i)|$ is the entry number of the $i$-th channel in the tensor $\mathcal{I}$. The smaller $s_i$ is, the less important the corresponding filter is.
\end{itemize}
We also compare the proposed SSR scheme to the state-of-the-art filter pruning methods, including SSL \cite{wen2016learning}, ThiNet \cite{luo2017ThiNet}, Taylor expansion (TE \cite{molchanov2017pruning}), CGES \cite{yoon2017combined} and GSS \cite{torfi2018attention}. Furthermore, we make a comparison with our alternative schemes with different regularizers, \emph{i.e.,} SSR with $\ell_{2,1}$-norm (termed SSR-L2,1), SSR with $\ell_{2,0}$-norm (termed SSR-L2,0) and SSR with $\ell_1$-norm (termed SSR-L1).


\begin{figure*}[t]
\centering
  \subfigure{
    \includegraphics[scale = 0.45]{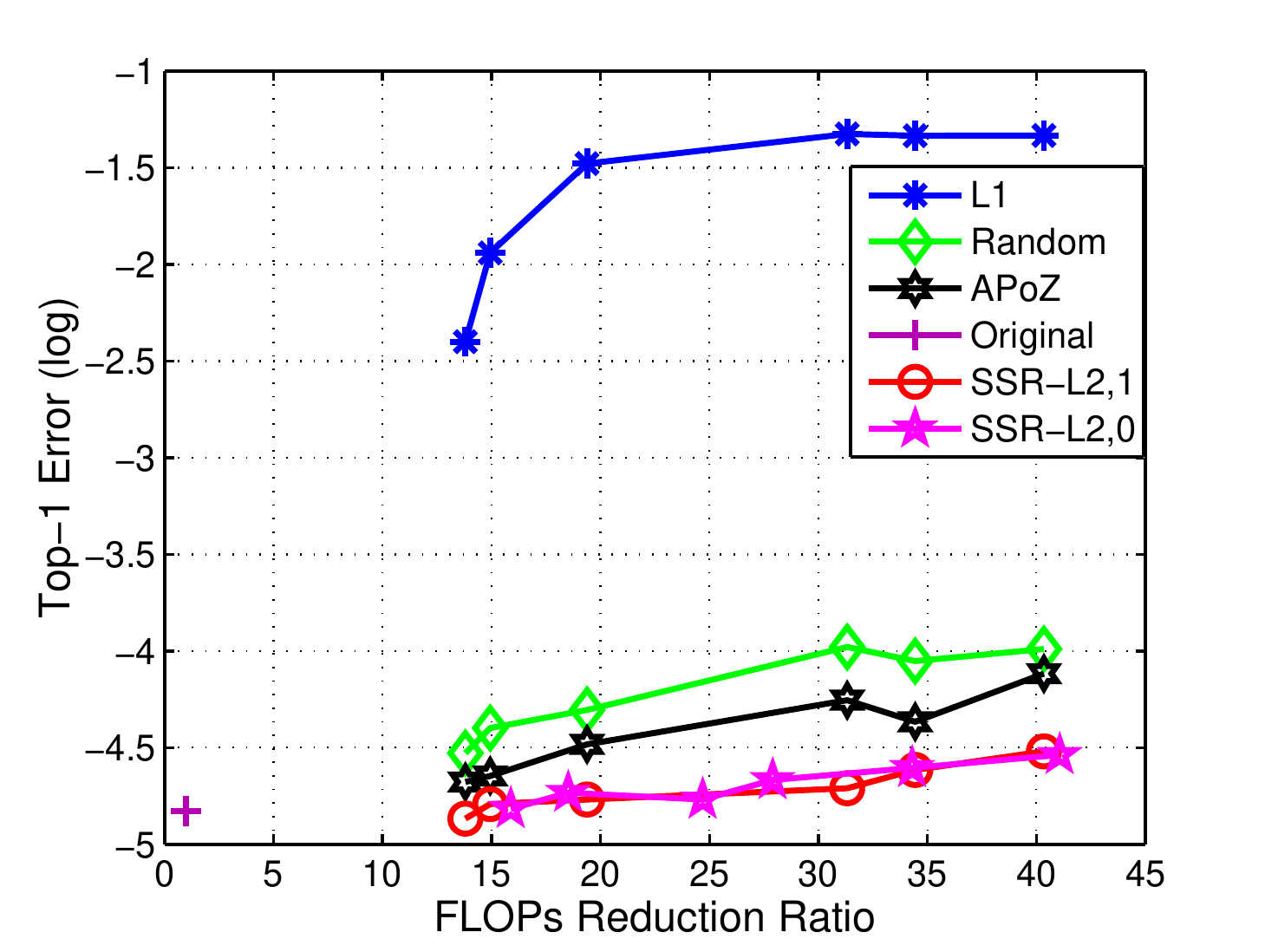}
    \label{fig:3_a}
    }
  \subfigure{
    \includegraphics[scale = 0.45]{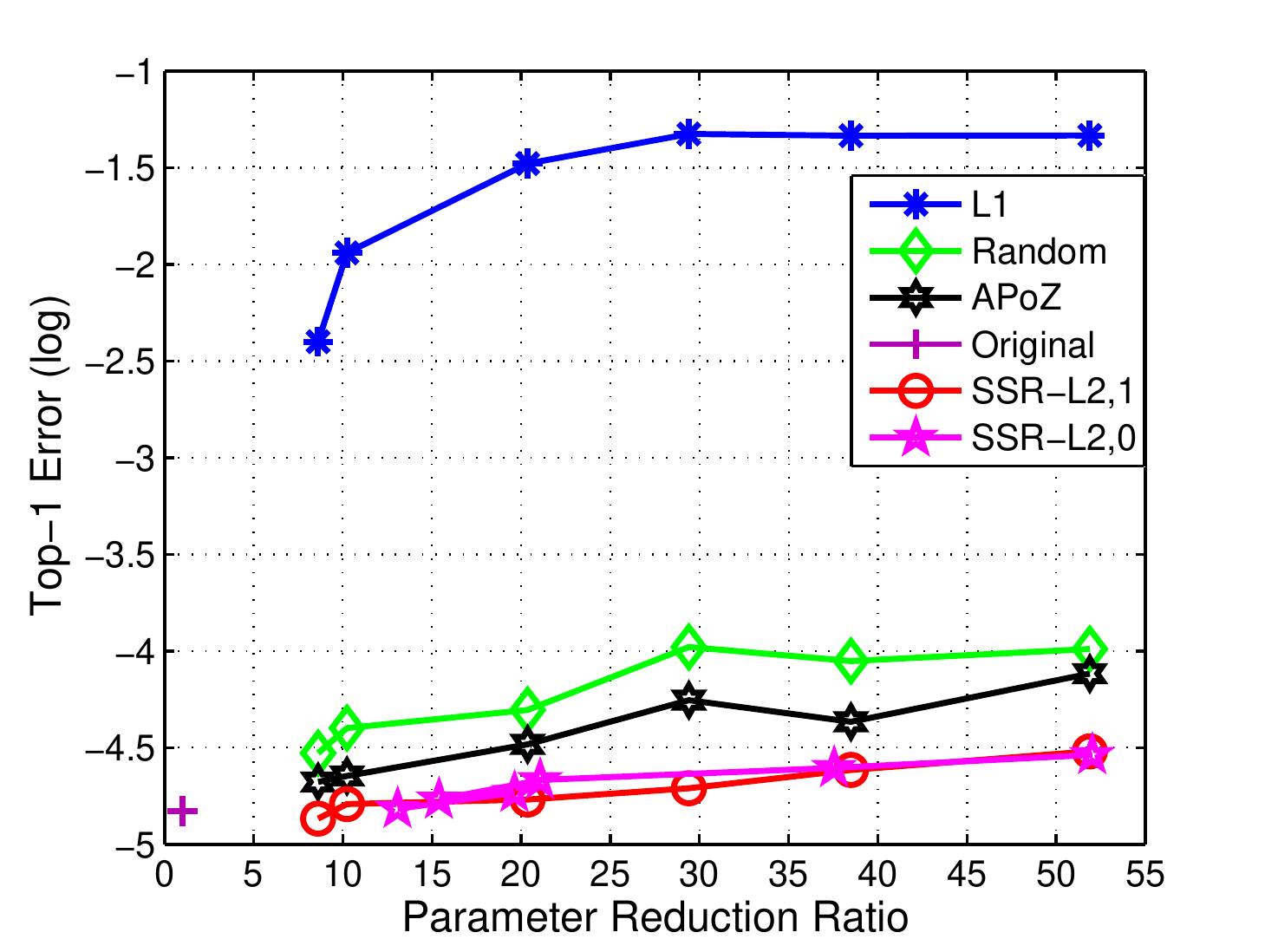}
    \label{fig:3_b}
  }
 \vspace{-1em}
 \caption{The results of different evaluation criteria on FLOPs and parameter numbers for compressing LeNet.} 
 \label{fig3}
\end{figure*}

\subsection{LeNet on MNIST}
MNIST is a small-scale dataset, which contains a training set of 60,000 images and a test set of 10,000 images from 10 classes. Each image is a $28 \times 28$ gray-scale handwritten digit image. LeNet on MNIST consists of 2 convolutional layers and 2 fully-connected layers, which achieves an error rate of 0.88\% on MNIST.
The detailed structure of LeNet is:
\begin{equation}
\label{eq16}
(20)C5 - MP2 - (50)C5 - MP2 - 500FC - 10FC - S,
\end{equation}
where $(20)C5$ is a $5\times 5$ convolutional layer with 20 filters, $MP2$ is a max-pooling layer with kernel size 2, $10FC$ is a fully-connected layer with 10 nodes, $S$ is a softmax loss layer.
Since the node number of the last layer is directly related to the number of classes, we prune the remaining layers except for the last layer\footnote{For AlexNet, VGG-16 and ResNet-50, we also keep the number nodes of the final layer unchanged and prune other remaining layers.}. The regularization factor $\lambda$ is fixed on 6 groups, \emph{i.e.}, (0.1,0.1,0.3), (0.3,0.3,0.4), (0.3,0.5,0.5), (0.4,0.5,0.5), (0.5,0.4,0.5), and (0.5,0.5,0.5). We compare our method to different criteria of filter selection \cite{li2017pruning,hu2016network} and also to other state-of-the-art methods in filter pruning \cite{wen2016learning,molchanov2017pruning,torfi2018attention,yoon2017combined}.

Fig.\,\ref{fig3} shows the pruning results of LeNet based on different filter selection criteria. Compared to these criteria and baselines, both FLOPs and parameter sizes are significantly reduced with lower error by using the proposed structured sparsity regularization, as expected. 
In contrast, the $\ell_1$-norm based filter selection performs poorly. 
To explain, due to the nonlinear transformation in the network, filters with small $\ell_1$-norm are still likely to be important, which have large impact on the final loss function. 
When a large proportion of filters with small $\ell_1$-norm are pruned, the classification error can be significantly increased.
Note that the simplest scheme of random filter selection works reasonably well, which is due to the self-recovery ability of the distributed representations \cite{bengio2013representation}. 
However, the random criterion is not robust in practice and may lead to large accuracy loss when being applied to compress large network (\emph{e.g.}, VGG-16), as presented in Table\,\ref{tab4} and Table\,\ref{tab5}.
For APoZ, the sparsity of feature maps is quite reasonable to prune the redundant filters, which is due to the self-sparsity of the pre-trained model with ReLU activation. 
In contrast, compared to these filter pruning methods, our SSR-L2,1 achieves the best performance with an increase of 0.29\% classification error, $40.37\times$ FLOPs reduction, and $51.89\times$ parameter reduction. 
SSR-L2,0 achieves relatively consistent results with SSR-L2,1. 
In particular, with a significant high compression ratio (\emph{i.e.,} $52.03\times$ parameter reduction ratio), SSR-L2,0 achieves much better performance than SSR-L2,1, as the $\ell_{2,0}$-norm is used to directly measure the cardinality of the filter structure. 
%

The quantitative performance for compressing LeNet using the proposed scheme is further shown in Table\,\ref{tab1}. 
First, we found that the FLOPs do not directly reflect the actual speedup ratio in online inference. 
For instance, compared to LeNet, the proposed SSR-L2,1 can reach $19.5\times$ FLOPs reduction, with a $4\times$ actual speedup ratio. 
To explain, memory accesses for both inter-layer and intra-layer can significantly increase the computation consumption. 
Second, TE \cite{molchanov2017pruning} achieves the best trade-off between the speedup ratio and the classification error among all the baselines (\emph{e.g.,} SSL \cite{wen2016learning}, CGES \cite{yoon2017combined} and GSS \cite{torfi2018attention}), as it inherits the effectiveness of filter selection by estimating the loss increase of pruning each filter with Taylor expansion. Note that CGES+ \cite{yoon2017combined} is the combination of iterative pruning \cite{han2015deep} and CGES for further compressing LeNet, which achieves 0.04\% increase of classification error using 10\% parameters of the full network. 
Compared to CGES+, the proposed SSR-L1 with $\ell_1$-norm achieves a significant higher compression ratio by $47.78\times$ (\emph{i.e.,} 9K number of parameters), only with an increase of 0.12\% Top-1 error. However, there are no structural constraints on filters/weights, which leads to very low speedup under the same hardware/software evaluation environment\footnote{To make a fair comparison, we evaluate the actual speedup of SSR-L1 without special hardware/software accelerators in experiments.}. 
Third, by using the proposed AULM with structured filter sparsity to adaptively select and prune the redundant filters, SSR-L2,1 achieves the best trade-off between the speedup/compression ratio and the classification error. For example, the Top-1 error is only increased by 0.18\% with $5.5\times$ speedup and $39.09\times$ compression.

\begin{table}[t]
\scriptsize
\begin{center}
\caption{The evaluations of the number of parameters and FLOPs both in convolutional and fully-connected layers, computational time on CPU (ms), GPU (ms), and classification error rates (Top-1/5 Err.) of AlexNet, VGG-16, ResNet-50 and GoogLeNet with mini-batch size 24.}
\vspace{-0.5em}
\label{tab2}
\begin{tabular}{|c|c|c|c|c|c|c|}
\hline
\multirow{2}*{Model} & \multirow{2}*{\#Param.} & \multirow{2}*{FLOPs} & \multicolumn{1}{c|}{CPU} & \multicolumn{1}{c|}{GPU} & \multicolumn{1}{c|}{Top-1} & \multicolumn{1}{c|}{Top-5} \\
& & & (ms) & (ms) & Err. & Err. \\ \hline\hline
AlexNet & 62.4M & 729.7M & 2,194 & 30 & 43.40\% & 19.88\% \\\hline
VGG-16 & 138.4M & 15.5B & 8,287 & 266 & 31.66\% & 11.55\% \\\hline
ResNet-50 & 25.5M & 3.9B & 7,790 & 255 & 24.88\% & 7.70\% \\\hline
GoogLeNet & 7.0M & 1.6B & 5,480 & 136 & 31.90\% & 11.45\% \\
\hline
\end{tabular}
\end{center}
\end{table}

\subsection{ImageNet}
ImageNet 2012 contains over 1 million training images from 1,000 object classes, as well as a validation set of 50,000 images. 
Each image is rescaled to a size of $256\times 256$. A $224\times 224$ image is randomly cropped from each scaled image (except for AlexNet with a $227\times 227$ cropping size) and mirrored for data augmentation. 
We test the pruned network on the validation set using single-view testing (central patch only) to evaluate the classification accuracy.

We implement the proposed SSR scheme on four CNNs, \emph{i.e.}, AlexNet, VGG-16, ResNet-50 and GoogLeNet. 
AlexNet contains 5 convolutional layers and 3 fully-connected layers, VGG-16 contains 13 convolutional layers and 3 fully-connected layers, ResNet-50 contains 54 convolutional layers with 16 residual blocks, and GoogLeNet contains 21 convolutional layers with 9 inception blocks. 
Unlike AlexNet and VGG-16, ResNet-50 and GoogLeNet use the global average pooling over the last convolutional layer to reduce the number of parameters, which removes 3 fully-connected layers. 
The computation time and storage overhead of four networks, together with their classification error, are shown in Table\,\ref{tab2}. 

\begin{figure*}[!t]
\centering
  \subfigure{
    \includegraphics[scale = 0.32]{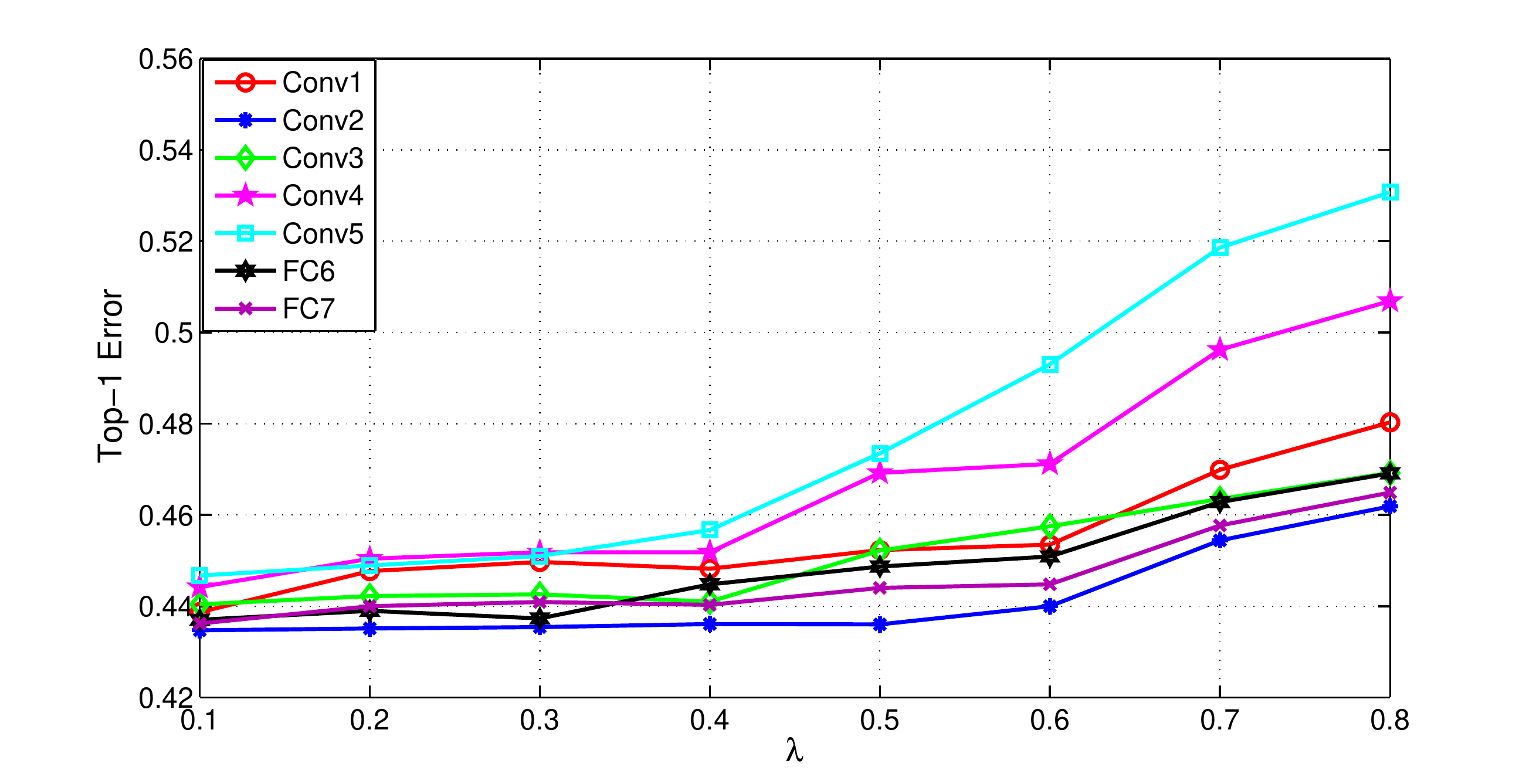}
    \label{fig:4_a}
    } 
  \subfigure{
    \includegraphics[scale = 0.32]{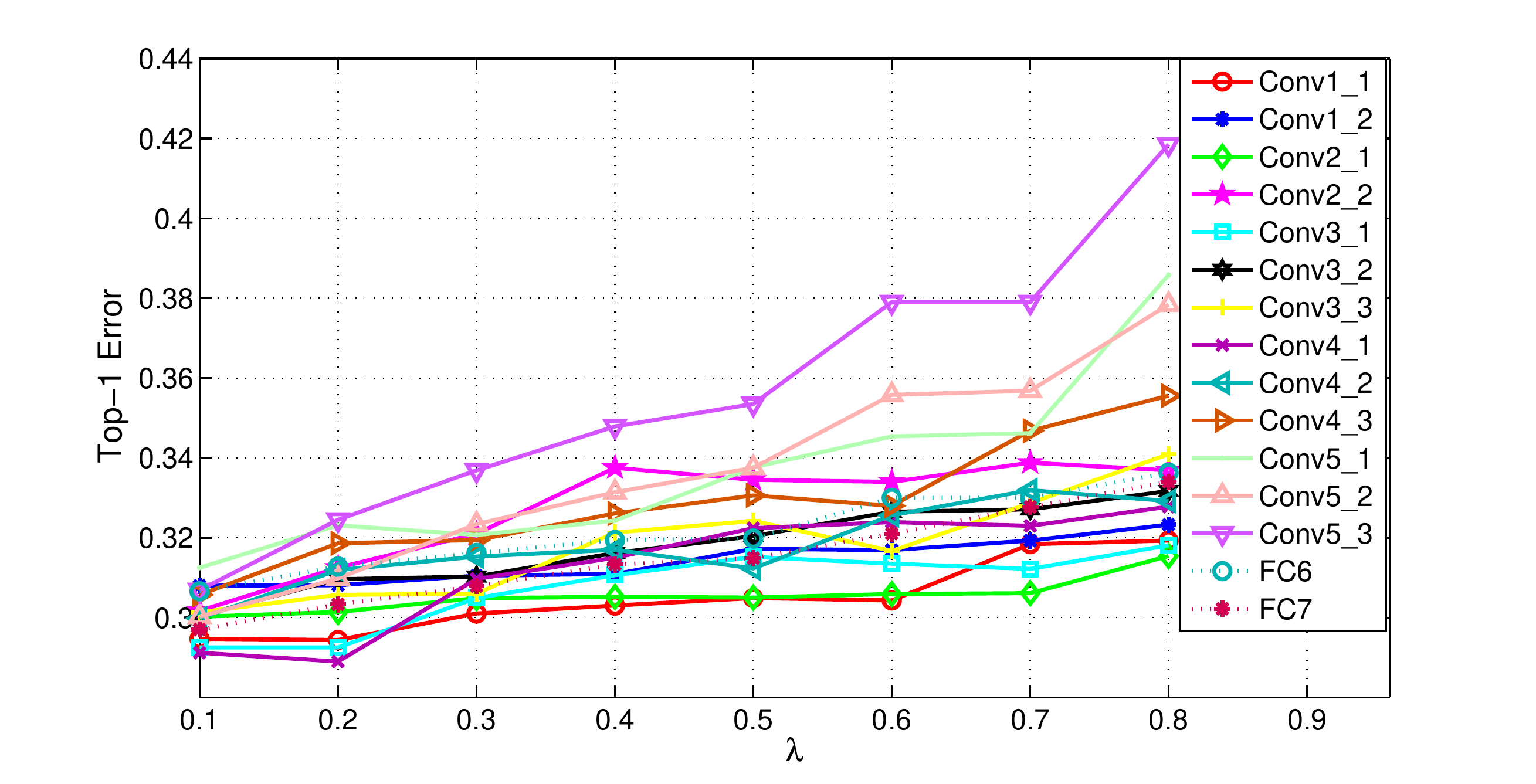}
    \label{fig:4_b}
  }
 \vspace{-1em}
 \caption{Sensitivity of pruning filter in each layer. Left: the sensitivity of AlexNet, Right: the sensitivity of VGG-16.} 
 \label{fig4}
\end{figure*}

\textbf{Sensitivity Analysis. }We explore the sensitivity of each layer in the network to guide the filter pruning for this layer. 
We take AlexNet and VGG-16 for instance, most layers are robust to prune, as
shown in Fig.\,\ref{fig4}. 
But still, there exist a small amount of sensitive layers, which locate at the top convolutional layers.
For example, it is sensitive to prune the $4$-th and the $5$-th convolutional layers for AlexNet and the last 3 convolutional layers (\emph{i.e.}, $Conv5\_1, Conv5\_2, Conv5\_3$) for VGG-16. To explain, these top layers often have high-level semantic information that is necessary for maintaining the classification accuracy. 
In addition, pruning some filters in the specific layers (\emph{e.g.}, $Conv1\_1, Conv3\_1, Conv4\_1$ in VGG-16 when $\lambda$ is set to be 0.2) gets a slightly better accuracy compared to the original network, which reveals that the redundant filters can reduce the discriminability of the original network.
Therefore, we use a different $\lambda$ for each layer to reduce the impact of these sensitive layers (\emph{i.e.}, we set a large $\lambda$ in insensitive layers, and set a small one in sensitive layers).   

\begin{table}[t]
\small
\begin{center}
\caption{Pruning results of AlexNet. The test mini-batch size is 24.}
\label{tab4}
\begin{tabular}{|c|c|c|c|c|c|}
\hline
\multirow{2}*{Method}   & \multirow{2}*{\#Param.} & \multicolumn{2}{c|}{Speedup} & \multicolumn{1}{c|}{Top-1} & \multicolumn{1}{c|}{Top-5} \\
\cline{3-4}
& & CPU & GPU & Err. $\uparrow$ & Err. $\uparrow$\\ \hline\hline
TE \cite{molchanov2017pruning} & 60.2M & $1.8\times$ & $1.4\times$ & 1.87\% & 1.59\% \\ \hline
\multirow{2}{*}{SSL \cite{wen2016learning}} & 60.9M & $1.5\times$ & $1.2\times$ & 1.32\% & 1.24\% \\ \cline{2-6}
 & 58.2M & $1.5\times$ & $1.3\times$ & 3.26\% & 2.48\%\\ \hline
\multirow{2}{*}{Random} & 60.8M & $1.6\times$ & $1.3\times$ & 2.25\% & 1.92\%\\ \cline{2-6}
 & 48.0M & $2.0\times$ & $1.5\times$ & 6.39\% & 4.66\%\\ \hline
\multirow{2}{*}{L1 \cite{li2017pruning}}  & 60.8M & $1.6\times$ & $1.3\times$ & 1.17\% & 0.86\%\\ \cline{2-6}
 & 48.0M & $2.0\times$ & $1.5\times$ & 3.36\% & 2.69\%\\ \hline
\multirow{2}{*}{APoZ \cite{hu2016network}}  & 60.8M & $1.6\times$ & $1.3\times$ & 2.17\% & 1.67\%\\ \cline{2-6}
 & 48.0M & $2.0\times$ & $1.5\times$ & 4.31\% & 2.75\%\\ \hline 
\multirow{2}{*}{SSR-L1} & 60.7M & $1.0\times$ & $1.0\times$ & 0.79\% & 0.68\%\\ \cline{2-6}
 & 22.6M & $1.2\times$ & $1.1\times$ & 1.58\% & 1.38\%\\ \hline
\multirow{2}{*}{SSR-L2,1} & 60.8M & $1.6\times$ & $1.3\times$ & 0.19\% & 0.20\%\\ \cline{2-6}
 & 48.0M & $2.0\times$ & $1.5\times$ & 1.72\% & 1.38\%\\ \hline 
\multirow{2}{*}{SSR-L2,0} & 60.8M & $1.6\times$ & $1.3\times$ & -0.05\% & 0.14\%\\ \cline{2-6}
 & 45.9M & $2.1\times$ & $1.7\times$ & 1.57\% & 1.28\%\\ \hline 
SSR-L2,1-GAP & 2.5M & $1.8\times$ & $1.4\times$ & 5.25\% & 4.62\%\\ \hline
SSR-L2,0-GAP & 2.5M & $1.8\times$ & $1.5\times$ & 5.12\% & 4.63\%\\ \hline
\end{tabular}
\end{center}
\end{table}

\textbf{Quantitative Results. }
Since the fully-connected layers occupy over 90\% storage in AlexNet and VGG-16, we replace the original fully-connected layers with global average pooling (GAP) \cite{lin2014network} to further compress the whole network. 
``X-GAP'' refers to the model using the GAP after all convolutional layers are pruned via ``X'' methods (\emph{e.g.}, ThiNet, SSR). The ``X-GAP'' is fine-tuned with the same fine-tuning parameters as described in Sec.\,\ref{ssec_4_1}.  

As shown in Table\,\ref{tab4}, we prune AlexNet with three groups of $\lambda$, \emph{i.e.}, (0.2, 0.4, 0.5, 0.6, 0.1, 0.1, 0.3), (0.4, 0.5, 0.7, 0.6, 0.1, 0.3, 0.3) and (0.2, 0.4, 0.5, 0.6, 0.1, GAP). 
Compared to other filter pruning methods \cite{li2017pruning,wen2016learning,molchanov2017pruning,hu2016network}, our SSR scheme achieves the best trade-off between the speedup/compression rate and the Top-1/5 classification error. 
First, we compare SSR-L2,1 to three alternative selection criteria (\emph{i.e.}, random, L1-norm \cite{li2017pruning} and APoZ \cite{hu2016network}) with the same pruning number in each layer\footnote{To make a fair comparison, the number of filter pruning in each layer based on three alternative selection criteria is the same to SSR-L2,1.}. SSR-L2,1 achieves the lowest Top-1/5 classification error. 
To explain, all selection criteria are naive methods to prune the filters based on the statistical property, resulting in a large approximation error of each layer that is propagated throughout the network.
Second, by directly employing SGD with filter-wise sparsity to solve the SSL problem, the redundant filters cannot be pruned efficiently, which only achieves $1.5\times$ CPU speedup with an increase of 1.32\% Top-1 error\footnote{It is different to the result reported in Wen \emph{et al.} \cite{wen2016learning}, due to the different fine-tuning framework and deep learning library.}.
Third, the work in \cite{molchanov2017pruning} uses Taylor expansion (TE) to approximate the loss increase, which is similar to ours but is with a totally different selection criterion. Quantitatively, it is time-consuming for pruning one filter and then fine-tuning the network iteratively. 
In contrast, SSR-L2,1 achieves the lowest Top-1 error increase of 1.72\% and Top-5 error increase of 1.38\%, while reducing much larger amount of parameters.
Fourth, we also compare two different kinds of structured sparsity regularizations (\emph{i.e.}, $\ell_{2,1}$-regularization and $\ell_{2,0}$-regularization) with element-wise sparsity regularization (\emph{i.e.}, $\ell_1$-regularization), and observe that SSR-L1 significantly reduces the memory storage with only 22.6M parameters, which is twice less than SSR-L2,1 and SSR-L2,0 with a comparable error increase. However, SSR-L1 does not boost the inference efficiency, as element-wise sparsity cannot significantly reduce the number of filters that the computation is on par with the full network (see later in Sec. \ref{ana} for more detailed discussions). 
As for structured sparsity regularization, compared to SSR-L2,1, SSR-L2,0 achieves a lower error increase (\emph{i.e.}, 1.57\% \emph{vs.} 1.72\% in Top-1 error increase) with a much smaller amount of parameters and higher speedup, \emph{i.e.}, 45.9M parameters and $2.1\times$ CPU speedup \emph{vs.} 48M parameters and $2.0\times$ CPU speedup.
Moreover, to further compress AlexNet using SSR, GAP enables the prune network to be more compact, leading to a $24.96\times$ compression rate (\emph{i.e.}, 2.5M parameters).

\begin{table}[t]
\scriptsize
\begin{center}
\caption{Pruning results of VGG-16. The test mini-batch size is 24.}
\vspace{-1.5em}
\label{tab5}
\begin{tabular}{|c|c|c|c|c|c|c|}
\hline
\multirow{2}*{Method}  & \multirow{2}*{FLOPs} & \multirow{2}*{\#Param.} & \multicolumn{2}{c|}{Speedup} & \multicolumn{1}{c|}{Top-1} & \multicolumn{1}{c|}{Top-5} \\
\cline{4-5}
& & & CPU & GPU & Err. $\uparrow$ & Err. $\uparrow$\\ \hline\hline
Random & 4.5B & 126.7M & $2.0\times$ & $2.4\times$ & 0.98\% & 0.43\%\\ \hline
L1 \cite{li2017pruning} & 4.5B & 126.7M & $2.0\times$ & $2.4\times$ & 0.29\% & -0.05\%\\ \hline
APoZ \cite{hu2016network} & 4.5B & 126.7M & $2.0\times$ & $2.4\times$ & -0.64\% & -0.43\%\\ \hline
TE \cite{molchanov2017pruning} & 4.2B & 135.7M & - & $2.7\times$ & - & 3.94\%\\ \hline
ThiNet \cite{luo2017ThiNet} & 5.0B & 131.5M & $1.9\times$ & $2.2\times$ & -1.46\% & -1.09\%\\ \hline
SSR-L2,1 & 4.5B & 126.7M & $2.0\times$ & $2.4\times$ & -1.46\% & -1.08\%\\ \hline
SSR-L2,0 & 4.5B & 126.2M & $2.0\times$ & $2.4\times$ & -1.65\% & -0.97\%\\ \hline\hline
Random-GAP & 4.4B & 9.2M & $2.1\times$ & $2.5\times$ & 5.47\% & 4.39\%\\ \hline
L1-GAP \cite{li2017pruning} & 4.4B & 9.2M & $2.1\times$ & $2.5\times$ & 4.62\% & 3.10\%\\ \hline
APoZ-GAP \cite{hu2016network} & 4.4B & 9.2M & $2.1\times$ & $2.5\times$ & 3.72\% & 2.65\%\\ \hline
ThiNet-GAP \cite{luo2017ThiNet} & 4.9B & 9.5M & $2.0\times$ & $2.3\times$ & 1.00\% & 0.52\%\\ \hline 
SSR-L2,1-GAP & 4.4B & 9.2M & 2.1$\times$ & $2.5\times$ & 0.52\% & 0.27\%\\ \hline
SSR-L2,0-GAP & 4.3B & 9.0M & 2.1$\times$ & $2.6\times$ & 0.83\% & 0.62\%\\ \hline
\end{tabular}
\end{center}
\end{table}

For VGG-16, we summarize the performance comparison to \cite{luo2017ThiNet,li2017pruning,molchanov2017pruning,hu2016network} in Table\,\ref{tab5}.
In experiments, $\lambda$ in the first 10 convolutional layers is set to be (0.5, 0.4, 0.5, 0.3, 0.5, 0.3, 0.3, 0.4, 0.5, 0.3) with large values, while the last 3 convolutional layers are all set to be 0.1. 
$\lambda$ is set to be (0.1, 0.6) in the fully-connected layers. 
First, instead of directly pruning filters, ThiNet \cite{luo2017ThiNet} conducts a greedy local channel selection, while TE \cite{molchanov2017pruning} uses a greedy feature map selection to prune the feature maps. 
Compared to ThiNet, TE achieves a higher GPU speedup (\emph{i.e.}, $2.7\times$ \emph{vs.} $2.2\times$ in ThiNet), but has a significant increase in Top-5 error, which affects the discriminative ability of the compressed model.
For three alternative criteria of filter selection, APoZ \cite{hu2016network} achieves the lowest increase in both Top-1 and Top-5 classification error by the same factor of GPU and CPU speedup, \emph{e.g.}, $2.0\times$ CPU speedup and $2.4\times$ GPU speedup with a decrease of 0.64\% Top-1 classification error and 0.43\% Top-5 classification error, respectively.
Compared to all baselines with fully-connected layers, SSR-L2,1 achieves the best trade-off between classification error and speedup, \emph{e.g.}, an decrease of 1.46\% Top-1 error by a factor of $2.4\times$ GPU speedup. 
To explain, the relationship between the final output and the local filters is directly considered in SSR, which therefore can adaptively prune redundant filters that have less impact on the global outputs. 
Second, by replacing with GAP, the network is further compressed by a large rate, \emph{e.g.}, ThiNet-GAP achieves $14.6\times$ parameter reduction, \emph{i.e.}, 9.5M parameters \emph{vs.} 138.4M parameters in the full VGG-16. 
Compared to the above three selection criteria and ThiNet-GAP, our SSR-L2,1-GAP still achieves the best performance, \emph{i.e.}, $2.5\times$ GPU speedup and $15\times$ parameter reduction, only with an increase of 0.52\% Top-1 error. 
In addition, compared to SSR-L2,1-GAP, SSR-L2,0-GAP achieves a comparable result by a factor of $2.6\times$ GPU speedup and $15.4\times$ parameter reduction, with an increase of 0.83\% Top-1 error.

\begin{figure*}[!t]
\centering
  \subfigure[Epoch / \# of Parameters]{
    \includegraphics[scale = 0.4]{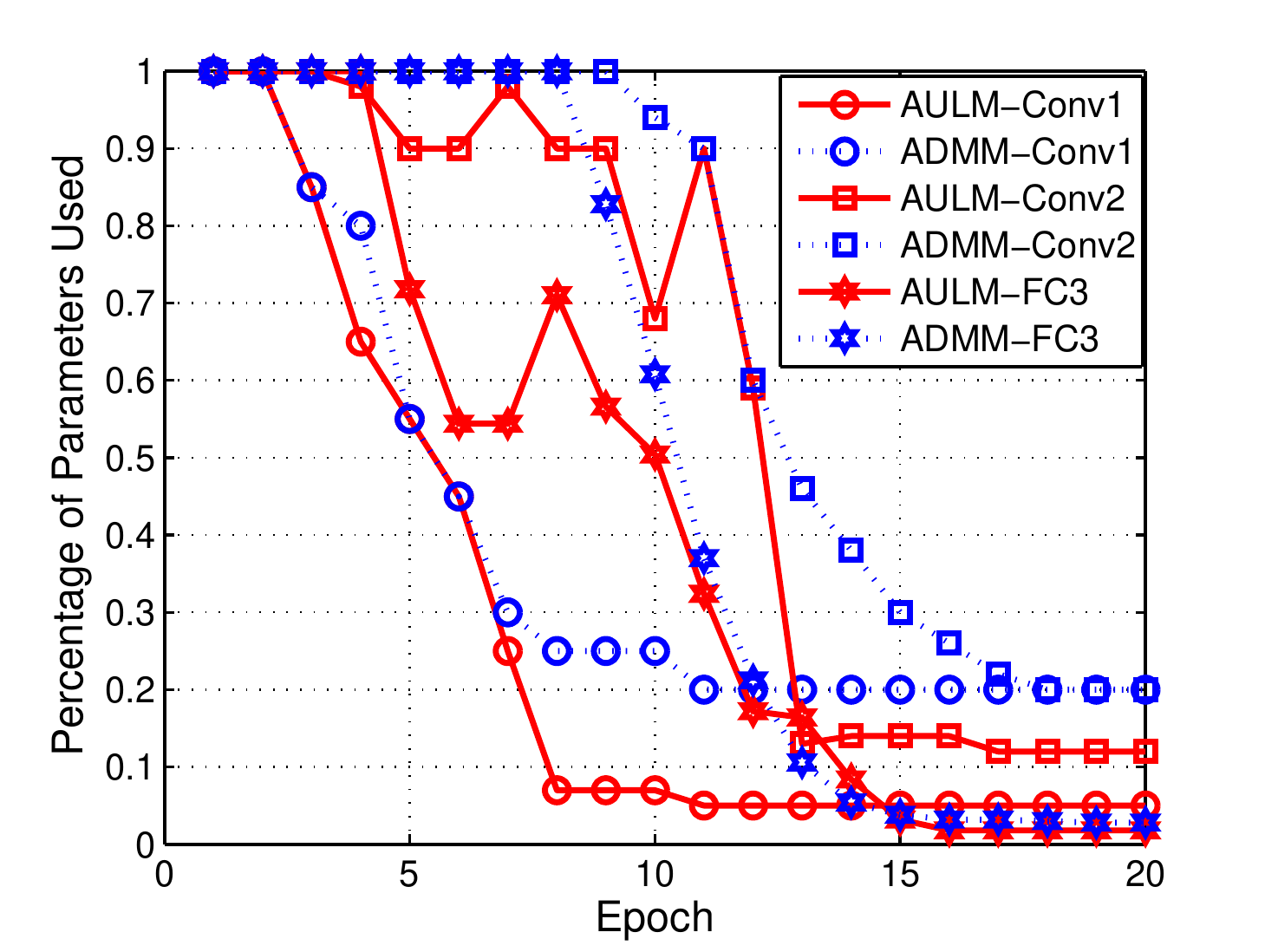}
    \label{fig:4_a}
    }	
  \subfigure[Epoch / Top-1 Error]{
    \includegraphics[scale = 0.4]{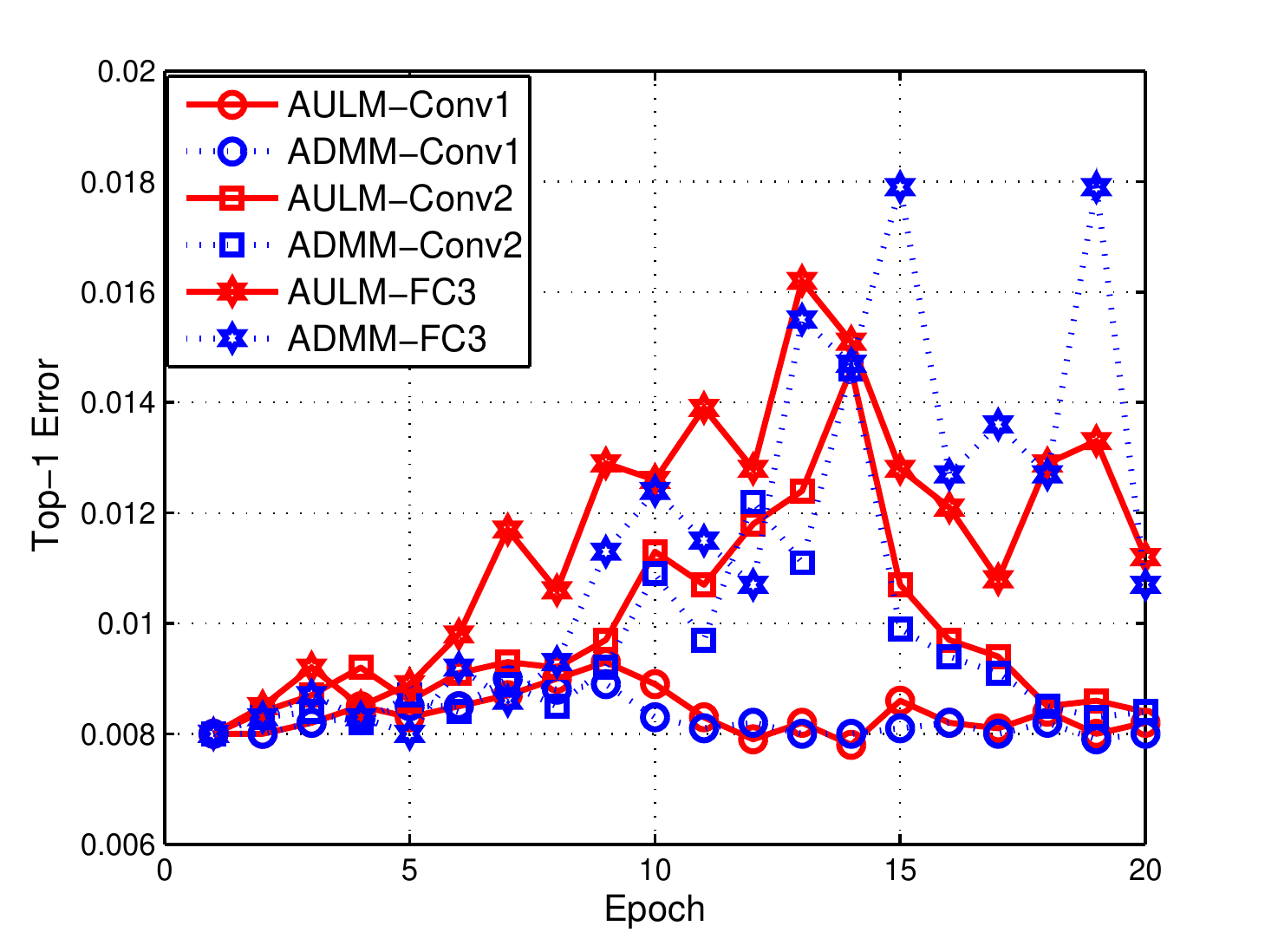}
    \label{fig:4_b}
    }
 \vspace{-0.5em}
 \caption{Further analysis of the $\ell_{2,1}$-regularization by our AULM and ADMM on MNIST dataset.} 
 \label{fig_aulm_admm}
\end{figure*}

\begin{table}[t]
\footnotesize
\begin{center}
\caption{Pruning results of ResNet-50. The test mini-batch size is 24.}
\vspace{-0.5em}
\label{tab6}
\begin{tabular}{|c|c|c|c|c|c|c|}
\hline
\multirow{2}*{Method}  & \multirow{2}*{FLOPs} & \multirow{2}*{\#Param.} & \multicolumn{2}{c|}{Speedup} & \multicolumn{1}{c|}{Top-1} & \multicolumn{1}{c|}{Top-5} \\
\cline{4-5}
& & & CPU & GPU & Err. $\uparrow$ & Err. $\uparrow$\\ \hline\hline
\multirow{2}{*}{ThiNet \cite{luo2017ThiNet}} & 2.4B & 16.9M & $1.7\times$ & $1.2\times$ & 3.09\% & 1.63\%\\ \cline{2-7} 
& 1.7B & 12.3M & $1.9\times$ & $1.4\times$ & 4.12\% & 2.28\%\\
\hline
SSL \cite{wen2016learning} & 2.1B & 13.2M & $1.7\times$ & $1.3\times$ & 4.58\% & 2.68\%\\ \hline
\multirow{2}{*}{Random} & 1.9B & 15.9M & $1.8\times$ & $1.3\times$ & 3.75\% & 2.49\%\\ \cline{2-7}
& 1.7B & 12.2M & $1.9\times$ & $1.4\times$ & 4.65\% & 2.75\%\\
\hline
\multirow{2}{*}{L1 \cite{li2017pruning}} & 1.9B & 15.9M & $1.8\times$ & $1.3\times$ & 3.36\% & 2.08\%\\ \cline{2-7}
& 1.7B & 12.2M & $1.9\times$ & $1.4\times$ & 4.31\% & 2.42\%\\
\hline
\multirow{2}{*}{APoZ \cite{hu2016network}} & 1.9B & 15.9M & $1.8\times$ & $1.3\times$ & 3.47\% & 2.39\%\\ \cline{2-7}
& 1.7B & 12.2M & $1.9\times$ & $1.4\times$ & 4.25\% & 2.41\%\\
\hline
\multirow{2}{*}{SSR-L2,1} & 1.9B & 15.9M & $1.8\times$ & $1.3\times$ & 2.99\% & 1.73\%\\ \cline{2-7}
& 1.7B & 12.2M & $1.9\times$ & $1.4\times$ & 3.97\% & 2.01\%\\
\hline 
\multirow{2}{*}{SSR-L2,0} & 1.9B & 15.5M & $1.9\times$ & $1.4\times$ & 2.83\% & 1.57\%\\ \cline{2-7}
& 1.7B & 12.0M & $1.9\times$ & $1.4\times$ & 3.65\% & 2.11\%\\
\hline
\end{tabular}
\end{center}
\end{table}

We also perform the proposed SSR on multi-branch networks, \emph{e.g.}, ResNet-50 and GoogLeNet. The results of SSR on ResNet-50 are shown in Table\,\ref{tab6}. We prune ResNet-50 with two groups of $\lambda$.
In the first 7 residual blocks of the first group, the hyper-parameter $\lambda$ of each residual block is set to be (0.4, 0.3), while $\lambda$ of each residual block is set to be (0.3, 0.4) in the remaining residual blocks (\emph{i.e.}, 9 residual blocks). At the second group, the corresponding $\lambda$ of each residual block is increased by 0.1 on the basis of the first group. 
Note that, we skip the first convolutional layer, which is pretty sensitive for pruning. 
In addition, we prune the first two layers in each residual block and leave the output and projection shortcuts of residual block unchanged, as shown in Fig.\,\ref{fig2}. 
We found that SGD in SSL \cite{wen2016learning} is not very effective to solve Eq.\,(\ref{eq2}) with $\ell_{2,1}$-regularization, which leads to a significant error increase with a limited FLOPs reduction. 
For three alternative criteria of filter selection, APoZ \cite{hu2016network} still achieves the lowest error increase at the same computation complexity and memory storage. 
Although ThiNet \cite{luo2017ThiNet} achieves the best performance among these state-of-the-art baselines \cite{li2017pruning,wen2016learning,hu2016network}, it requires additional samples (\emph{i.e.}, new input/output pairs from hidden layers) in each layer to find the optimal channels for pruning, which is not only expensive to store additional training datasets, and is also time consuming to collect them in offline training.
Moreover, ThiNet only reduces the reconstruction error of each layer, which however ignores the correlation between local filter pruning and global output, leading to the accumulation of reconstruction error.
Compared to ThiNet, without supervised information of hidden layers, SSR-L2,0 employs the original ImageNet dataset to improve the classification accuracy at the same speedup ratio, and also achieves a higher parameter reduction ($2.13\times$ \emph{vs.} $2.07\times$). 
For GoogLeNet, we prune all convolutional filters with high computation complexity,  \emph{i.e.,} filter size of $3\times3$ and $5\times5$. For $\lambda$, it is set to be 0.5 and 0.3 in the first three inception blocks and the remaining inception blocks, respectively. We skip the first convolutional layer and kernels with a size of $1\times1$ for effective pruning. As shown in Table\,\ref{tab7}, compared to three alternative criteria of filter selection, SSR-L2,1 achieves a lower error increase at the same pruning-level. By replacing $\ell_{2,1}$-regularization with $\ell_{2,0}$-regularization, SSR-L2,0 achieves the best performance with an increase of 1.05\% Top-5 error by $1.7\times$ compression and $1.6\times$ CPU speedup.

\begin{figure*}[!t]
\centering
  \subfigure[\# Iterations / Training Loss]{
    \includegraphics[scale = 0.385]{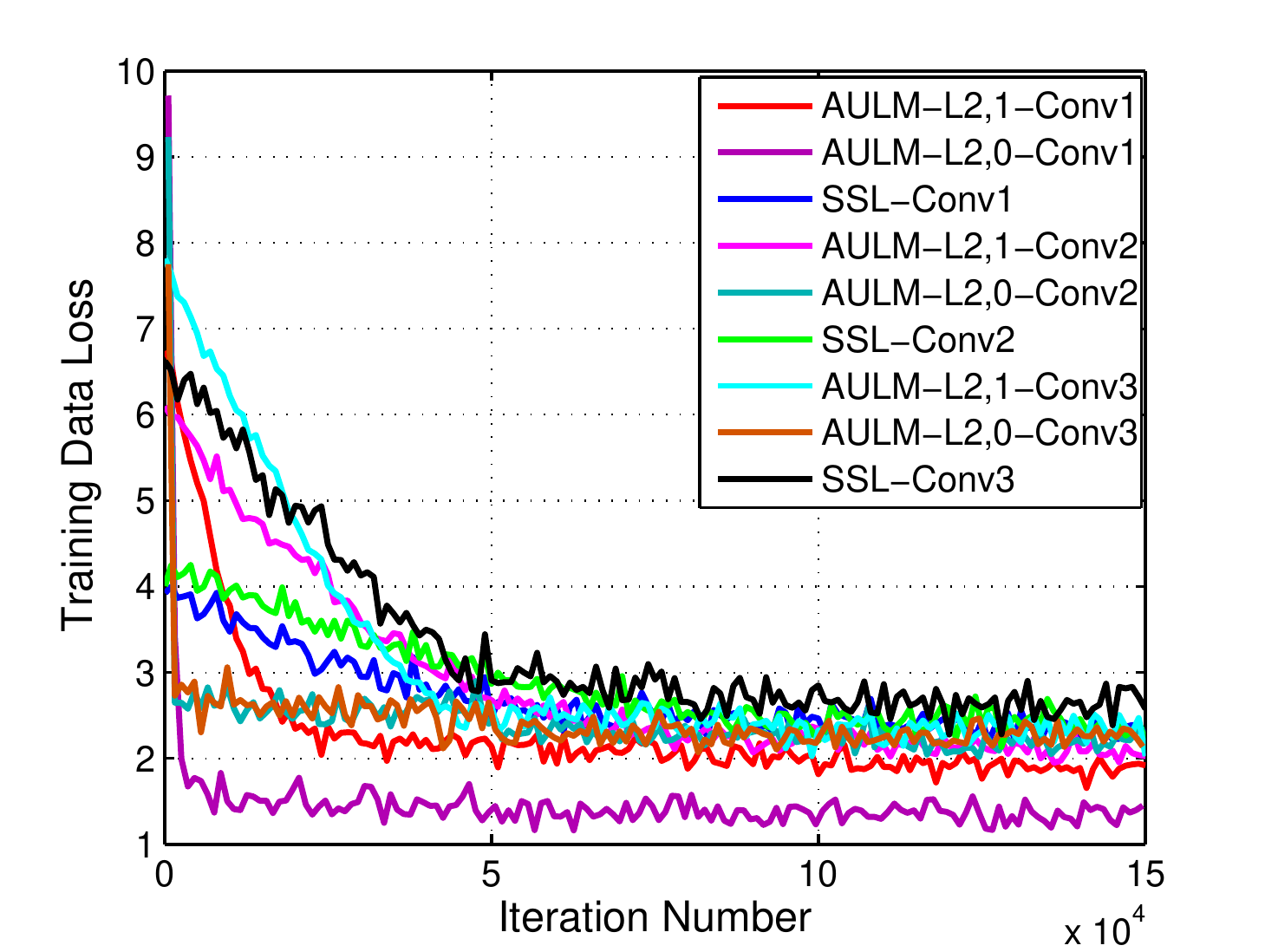}
    \label{fig:5_a}
    }	
  \subfigure[Epoch / \# of Pruned Filters]{
    \includegraphics[scale = 0.385]{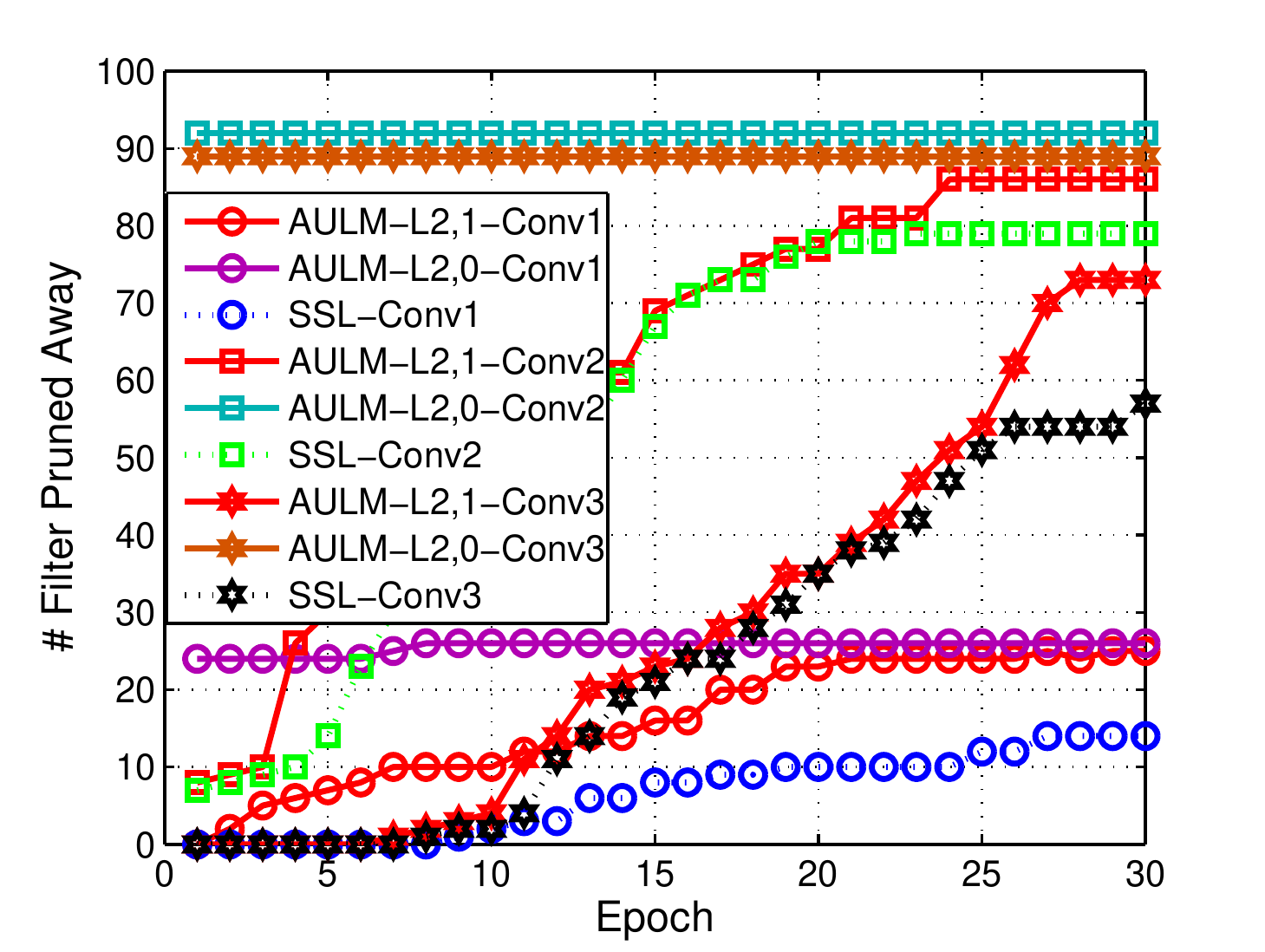}
    \label{fig:5_b}
    }
  \subfigure[Epoch / Top-5 Error]{
    \includegraphics[scale = 0.385]{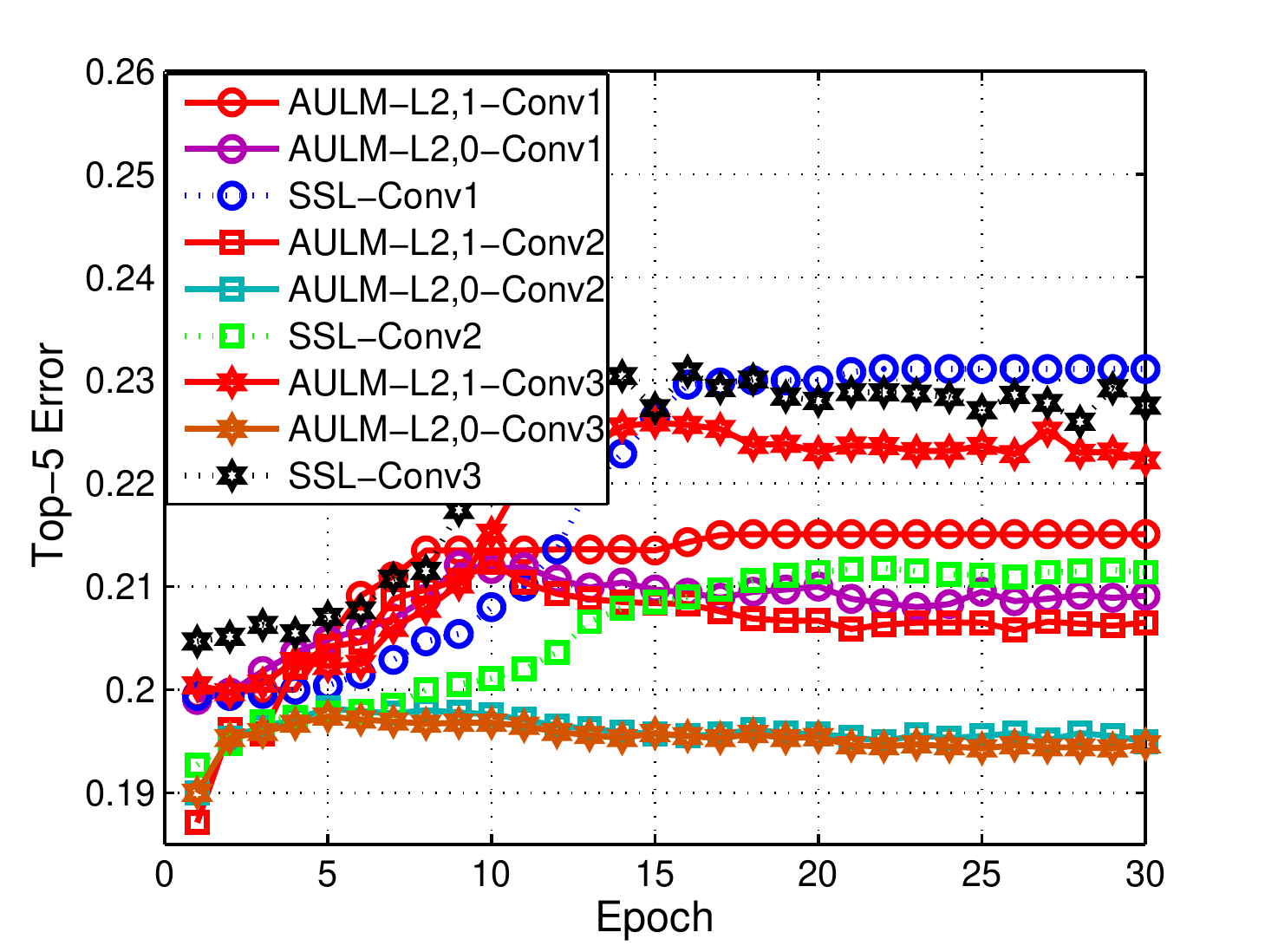}
    \label{fig:5_c}
  }
 \vspace{-1em}
 \caption{Convergence of our AULM slover and SSL in the first three convolutional layers of AlexNet.} 
 \label{fig5}
\end{figure*}

\begin{table}[t]
\footnotesize
\begin{center}
\caption{Pruning results of GoogLeNet. The test mini-batch size is 24.}
\vspace{-0.5em}
\label{tab7}
\begin{tabular}{|c|c|c|c|c|c|c|}
\hline
\multirow{2}*{Method}  & \multirow{2}*{FLOPs} & \multirow{2}*{\#Param.} & \multicolumn{2}{c|}{Speedup} & \multicolumn{1}{c|}{Top-1} & \multicolumn{1}{c|}{Top-5} \\
\cline{4-5}
& & & CPU & GPU & Err. $\uparrow$ & Err. $\uparrow$\\ \hline\hline
Random & 1.0B & 4.2M & $1.5\times$ & $1.4\times$ & 4.04\% & 2.54\%\\
\hline
L1 \cite{li2017pruning} & 1.0B & 4.2M & $1.5\times$ & $1.4\times$ & 3.00\% & 1.76\%\\
\hline
APoZ \cite{hu2016network} & 1.0B & 4.2M & $1.5\times$ & $1.4\times$ & 2.30\% & 1.29\%\\
\hline
SSR-L2,1 & 1.0B & 4.2M & $1.5\times$ & $1.4\times$ & 1.90\% & 1.26\%\\
\hline 
SSR-L2,0 & 1.0B & 4.1M & $1.6\times$ & $1.4\times$ & 1.81\% & 1.05\%\\
\hline
\end{tabular}
\end{center}
\end{table}


\subsection{Analysis}
\label{ana}
\textbf{Efficiency Analysis. }
We first analyze the empirical efficiency of AULM and ADMM. As shown in Fig.\,\ref{fig_aulm_admm}, we found that our AULM can help to learn a more compact network with almost the same error using much fewer epochs, compared to ADMM. For instance, AULM achieves 0.9\% error with only 8 epochs and 7\% parameters in the Conv1 layer, while ADMM achieves almost the same error but requires 11 epochs and 20\% parameters.
This faster convergence for structured filter sparsity is due to that we apply Nesterov's optimization to overrelax variables (\emph{i.e.}, structured sparse filters and dual variables) for accelerating the alternative optimization.
Second, we further study the influence of different optimization strategy (\emph{i.e.}, SSL \emph{VS.} AULM) on filter pruning. 
In SSL \cite{wen2016learning}, the structured sparsity of filters with $\ell_{2,1}$-norm is learned by directly solving Eq.\,(\ref{eq2}) with SGD, which is different from the proposed AULM-L2,1. 
Taking the first three layers of AlexNet for instance, they occupy a significant proportion of the computational overhead. 

Fig.\,\ref{fig5} presents the convergence process of different solvers over different layers in Conv1, Conv2, Conv3, respectively. 
Compared to SSL, as shown in Fig.\,\ref{fig:5_a}, AULM-L2,1 is faster to reduce the training loss and also achieves a lower training loss, which leads to more effective training for structured pruning.  
Moreover, by using AULM-L2,1 instead of SSL, the number of pruned filters is always larger, the corresponding classification error is lower, and the convergence is faster, especially in the first convolutional layer (the convergence after 8 epochs in AULM-L2,1 \emph{vs.} 20 epochs in SSL). 
Therefore, alternative optimization in AULM is more effective to prune the network than SGD in SSL.
Furthermore, we also make a comparison between two different structured regularizers (\emph{i.e.}, $\ell_{2,1}$-norm and $\ell_{2,0}$-norm) to explicitly analyze their convergence by AULM\footnote{We have not compared the SGD to our AULM on $\ell_{2,0}$-norm regularization, since SGD cannot solve the NP-hard problem.}. 
As shown in Fig.\,\ref{fig:5_b} and \ref{fig:5_c}, compared to AULM-L2,1, we observe that AULM-L2,0 is not only significantly faster to generate more structured filters, but also achieves a lower Top-5 error. Interestingly, the number of structured filters is almost constant during training, which is due to the closed-form solution by Eq.\,(\ref{eq12}) that leads to almost the same structured sparsity of intermediate filters after the first updating.

\begin{figure*}[!t]
\centering
  \subfigure[SSR-Ll, Error: 21.26\%, Sparsity: 58.67\%]{
    \includegraphics[scale = 1]{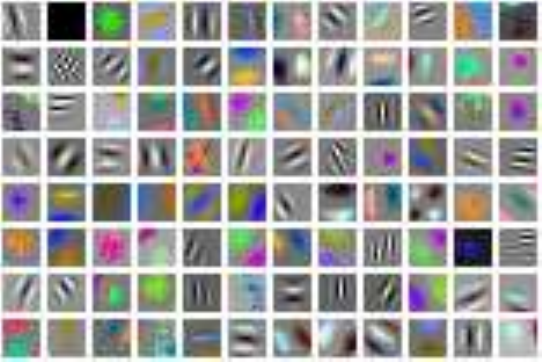}
    \label{fig:6_a}
    }	
  \subfigure[SSR-L2,1, Error: 21.26\%, Sparsity: 13.54\%]{
    \includegraphics[scale = 1]{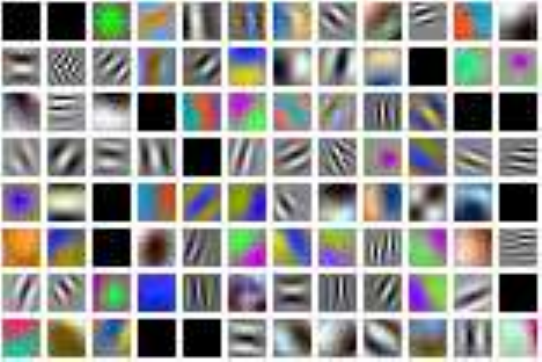}
    \label{fig:6_b}
    }
  \subfigure[SSR-L2,0, Error: 21.16\%, Sparsity: 14.58\%]{
    \includegraphics[scale = 1]{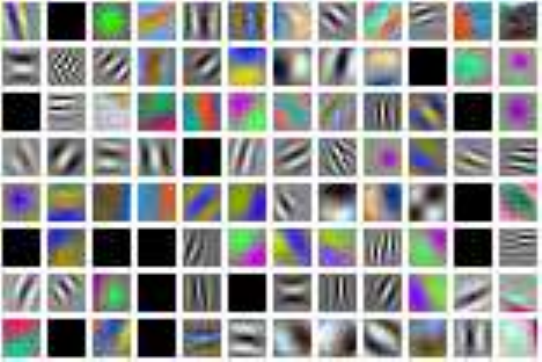}
    \label{fig:6_c}
  }
  \vspace{-0.5em}
 \caption{Visualizations of the first convolutional layer for pruning the whole AlexNet using SSR with different regularizers. (a) SSR with $\ell_1$-regularization results in element-wise sparse filters. (b) SSR with $\ell_{2,1}$-regularization results in filter-wise removal (filters with dark color). (c) SSR with $\ell_{2,0}$-regularization obtains more structured sparse filters, and achieves lower Top-5 error than SSR-L2,1.} 
 \label{fig6}
\end{figure*}

\textbf{Visualization.}
To verify the effectiveness of structured filter sparsity, we visualize the filters in the first convolutional layer of AlexNet using SSR with three different regularizers (\emph{i.e.}, $\ell_1$-norm, $\ell_{2,1}$-norm and $\ell_{2,0}$-norm), as shown in Fig.\,\ref{fig6}. Although SSR with $\ell_1$-regularization obtains a large amount of sparse filter elements, it is unable to remove the whole filter, which leads to a very limited speedup without the specialized software. In contrast, SSR with $\ell_{2,1}$-regularization or $\ell_{2,0}$-regularization results in complete filter removal, which directly accelerates the network inference. Compared to $\ell_{2,1}$-regularization, SSR with $\ell_{2,0}$-regularization achieves a lower Top-5 error with more structured sparse filters. This is due to that the $\ell_{2,0}$-regularization explicitly provides the natural constraint for filter selection.

%

\subsection{Generalization Ability for Transfer Learning.}
SSR has demonstrated its effectiveness to simutaneously accelerate and compress CNNs on MNIST and ImageNet 2012 classification tasks. We further investigate the generalization ability of the compressed model in transfer learning, including domain adaptation and object detection. For better discussion, we take VGG-16 as our baseline model. 

\subsubsection{Domain Adaptation}
Since SSR does not change the network structure, a model pruned on ImageNet can be easily transferred into other domains. To evaluate the domain adaptation ability of the compressed model, we consider a practical application in which we transfer the pruned model trained on ImageNet  into a smaller one by using a domain-specific dataset. To this end, we select a public domain-specific dataset CUB-200 \cite{WelinderEtal2010} for fine-grained classification to evaluate the ability of domain adaptation. CUB-200 contains 11,788 images of 200 different bird species, which contains 5,994 training images and 5,794 testing images. To make a fair comparison, we fine-tune the compressed models on ImageNet by Random, L1, APoZ and SSR with the same hyper-parameters and epochs\footnote{The related implementation details are described in the following: https://github.com/Roll920/fine-tune-avg-vgg16.}. The results of fine-grained classification are shown in Table\,\ref{tab_domain}.

\begin{table}[t]
\begin{center}
\caption{Comparison of different compressed models for fine-grained classification on CUB-200.}
\vspace{-0.5em}
\label{tab_domain}
\begin{tabular}{|c|c|c|c|}
\hline
Method & \#param. & FLOPs & Top-1 err. \\
\hline\hline
VGG-16 & 135.1M & 15.5B & 27.60\% \\ \hline
Random & 124.6M & 4.5B & 36.64\% \\ \hline
L1 \cite{li2017pruning} & 124.6M & 4.5B & 30.00\% \\ \hline
APoZ \cite{hu2016network} & 124.6M & 4.5B & 29.10\% \\ \hline
SSR-L2,1 & 124.6M & 4.5B & 28.70\% \\ \hline\hline
Random-GAP & 8.8M & 4.4B & 37.86\% \\ \hline
L1-GAP \cite{li2017pruning} & 8.8M & 4.4B & 33.20\% \\ \hline
APoZ-GAP \cite{hu2016network} & 8.8M & 4.4B & 32.14\% \\ \hline
SSR-L2,1-GAP & 8.8M & 4.4B & 29.55\% \\ \hline
\end{tabular}
\end{center}
\end{table}

The pre-trained VGG-16 is first fine-tuned on the CUB-200 dataset, which is an effective method to directly transfer the model from ImageNet domain to CUB-200 domain. As shown in Table\,\ref{tab_domain}, the pre-trained VGG-16 achieves the lowest error (27.60\% Top-1 error) but has a huge memory cost and slow inference speed (\emph{i.e.}, 135.1M parameters and 15.5B FLOPs). 
We then fine-tune the compressed networks, which are previously compressed in the ImageNet domain by Random, L1 \cite{li2017pruning}, APoZ \cite{hu2016network} and SSR, respectively. Compared to Random, L1 and APoZ, the model compressed by SSR-L2,1 achieves the best performance by an increase of only 1.1\% Top-1 error, with 124.6M parameters and 4.5B FLOPs. Furthermore, we also fine-tune the compressed models, in which the traditional fully-connected layers are replaced with GAP. We obtain a more compact model with 8.8M parameters and 4.4B FLOPs, \emph{i.e.}, 15.4$\times$ lower memory cost and theoretical 3.5$\times$ inference speedup than VGG-16. Compared to the three alternative selection criteria, SSR-L2,1-GAP achieves the lowest Top-1 error, \emph{i.e.}, 29.55\% Top-1 error, which is an increase of only 1.95\% Top-1 error.

\begin{table}[!t]
\small
\begin{center}
\caption{The speedup for Faster R-CNN detection.}
\vspace{-0.5em}
\label{tab_object_detection}
\begin{tabular}{|c|c|c|c|c|}
\hline
Device & Method & Speedup & mAP & $\Delta$ mAP \\
\hline \hline
\multirow{5}{*}{Titan X GPU} & VGG-16 & Baseline & 68.7 & - \\ \cline{2-5}
& Random & 2.45 & 67.9 & 0.8 \\ \cline{2-5}
& L1 \cite{li2017pruning} & 2.45 & 68.1 & 0.6 \\ \cline{2-5}
& APoZ \cite{hu2016network} & 2.45 & 67.0 & 1.7 \\ \cline{2-5}
& SSR-L2,1 & 2.45 & 68.4 & 0.3 \\ \hline
\end{tabular}
\end{center}
\end{table}

\subsubsection{Object Detection}
We also evaluate the ability of transfer learning for the compressed VGG-16 by Random, L1 \cite{li2017pruning}, APoZ \cite{hu2016network} and SSR with $\ell_{2,1}$-regularization, which are deployed over Faster R-CNN \cite{ren2015faster} for object detections. 
The PASCAL VOC 2007 object detection benchmark is selected to evaluate the performance of our models by mean Average Precision (mAP), which contains about 5K training/validation images and 5K testing images. 
In our experiments, we first compress VGG-16 by Random, L1, APoZ and SSR-L2,1 on ImageNet, and then use the compressed models as the pre-trained models for Faster-RCNN with the default training settings.

The actual running time of Faster R-CNN is 189ms/image on Titan X GPU. 
Compared to VGG-16, we get an actual detection time of 77ms with $2.45\times$ acceleration on Titan X. 
As shown in Table\,\ref{tab_object_detection}, interestingly, filter pruning by random criterion works interestingly well for object detection, which is due to the self-recovery ability in training the specific PASCAL VOC 2007 dataset. In contrary, APoZ may be unsuitable for object detection, which achieves the lowest mAP, \emph{i.e.}, 1.7\% mAP drops.
Compared to three alternative pruning criteria, SSR-L2,1 achieves the best performance by a factor of $2.45\times$ speedup on Titan X with only 0.3\% mAP drops, which is still very practical for real-world application.

\section{Conclusion}
In this paper, we propose a unified filter pruning scheme, termed structured sparsity regularization (SSR), for CNN acceleration and compression. SSR captures the relationship between global output and local filter pruning, forms a novel optimization problem based on two structured sparsity with $\ell_{2,1}$-norm and $\ell_{2,0}$-norm, both of which can be efficiently solved by a novel Alternative Updating with Lagrange Multipliers (AULM). The proposed AULM is fast to generate structured filters and is adaptive to prune redundant filters. We have demonstrated that the proposed SSR scheme achieves superior performance over the state-of-the-art filter pruning methods \cite{li2017pruning,hu2016network,wen2016learning,luo2017ThiNet,molchanov2017pruning}. We further evaluate the effectiveness of the compressed model by SSR when being applied to domain adaptation and objection detection.

In the future, we would like to investigate the specific design of filter pruning for ResNet and DenseNet, including (1) how to effectively prune the shortcut connection in the residual block, (2) design a more effective strategy for accelerating batch normalization and pooling layers, which are left unexploited in the existing works, and (3) design a novel filter selection layer to prevent the dimension mismatch of different dense blocks, which is due to the dense connectivity in the natural DenseNet.
\section*{Acknowledgment}

This work is supported by the National Key R\&D Program of China (No. 2017YFC0113000, No. 2016YFB1001503 and No. 2018YFB1107400), the Nature Science Foundation of China (No.U1705262, No. 61772443, No. 61402388, No. 61572410 and No. 61871470) 
, the Post Doctoral Innovative Talent Support Program under Grant BX201600094, the China Post-Doctoral Science Foundation under Grant 2017M612134 and the Nature Science Foundation of Fujian Province, China (No. 2017J01125).

\ifCLASSOPTIONcaptionsoff
  \newpage
\fi



%

\bibliographystyle{IEEEtran}
\bibliography{mybib}

\end{document}